\documentclass{article} 
\usepackage[preprint]{nips_2018}

\usepackage{times}

\usepackage{natbib}
\usepackage[utf8]{inputenc} 
\usepackage[T1]{fontenc}    
\usepackage{hyperref}       
\usepackage{url}            
\usepackage{booktabs}       
\usepackage{graphicx}
\usepackage{wrapfig,lipsum,booktabs}
\usepackage{subcaption}
\usepackage{amsmath}
\usepackage{tikz}
\usepackage{graphicx}
\usetikzlibrary{positioning}

\usepackage{amsfonts}
\usepackage{mathtools}
\usepackage{framed}
\usepackage{microtype}
\DeclareMathOperator*{\expectation}{\mathbb{E}}

\usepackage{epsfig}
\usepackage{url}

\usepackage{lipsum}

\usepackage{siunitx}
\usepackage{algorithm,algpseudocode}
\usepackage{xr-hyper}
\usepackage{xr}
\usepackage{xr}
\usepackage{wrapfig}
\usepackage{amsmath,amsthm}
\usepackage{amssymb}
\usepackage{epstopdf}
\usepackage{mathtools}

\newcommand{\manifoldmixup}{\textit{Manifold Mixup}}
\newcommand{\inputmixup}{Input Mixup}

\renewcommand\epsilon\varepsilon

\title{Manifold Mixup: Better Representations by Interpolating Hidden States}

\author{
  Vikas Verma* $\dagger$  \\
  Aalto Univeristy, Finland \hfill\\
  \texttt{vikas.verma@aalto.fi} \\
   \And
   Alex Lamb* \\
   Montr\'{e}al Institute for Learning Algorithms \\
   \texttt{lambalex@iro.umontreal.ca} \\
   \And
   Christopher Beckham \\
   Montr\'{e}al Institute for Learning Algorithms \\
   \texttt{christopher.j.beckham@gmail.com} \\
      \And
   Amir Najafi \hspace{1.4cm} \\
   Sharif University of Technology\hspace{1.4cm}  \\
   \texttt{najafy@ce.sharif.edu}\hspace{1.4cm} \\
   \AND
   Ioannis Mitliagkas \\
   Montr\'{e}al Institute for Learning Algorithms \\
   \texttt{imitliagkas@gmail.com} \\
   \And
   Aaron Courville\hspace{0.8cm} \\
   Montr\'{e}al Institute for Learning Algorithms \\
   \texttt{courvila@iro.umontreal.ca} \\
   \AND
   \hspace{0.8cm}David Lopez-Paz\\
   \hspace{0.8cm}Facebook AI Research\\
   \hspace{0.8cm}\texttt{dlp@fb.com} \\
   \And
   \hspace{0.5cm}Yoshua Bengio \\
   \hspace{0.5cm}Montr\'{e}al Institute for Learning Algorithms \\
   \hspace{0.5cm}CIFAR Senior Fellow \\
   \hspace{0.5cm}\texttt{yoshua.umontreal@gmail.com} \\
}

\renewcommand\epsilon\varepsilon

\begin{document}

\newtheorem{thm}{Theorem}[section]
\newtheorem{thm2}{Theorem}
\newtheorem{corl}{Corollary}
\newtheorem{note}[thm2]{Note}
\newtheorem{lemma}{Lemma}
\newtheorem{definition}{Definition}
\newtheorem{remark}{Remark}
\newtheorem{claim}{Claim}

\maketitle

\begin{abstract}

Deep neural networks excel at learning the training data, but often provide incorrect and confident predictions when evaluated on slightly different test examples.
This includes distribution shifts, outliers, and adversarial examples.
To address these issues, we propose \manifoldmixup{}, a simple regularizer that encourages neural networks to predict less confidently on interpolations of hidden representations.
\manifoldmixup{} leverages semantic interpolations as additional training signal, obtaining neural networks with smoother decision boundaries at multiple levels of representation. 
As a result, neural networks trained with \manifoldmixup{} learn class-representations with fewer directions of variance.
We prove theory on why this flattening happens under ideal conditions, validate it on practical situations, and connect it to previous works on information theory and generalization.
In spite of incurring no significant computation and being implemented in a few lines of code, \manifoldmixup{} improves strong baselines in supervised learning, robustness to single-step adversarial attacks, and test log-likelihood.
\end{abstract}

\section{Introduction}

Deep neural networks are the backbone of state-of-the-art systems for computer vision, speech recognition, and language translation \citep{lecun2015deep}.  
However, these systems perform well only when evaluated on instances very similar to those from the training set.  
When evaluated on slightly different distributions, neural networks often provide incorrect predictions with strikingly high confidence.
This is a worrying prospect, since deep learning systems are being deployed in settings where data may be subject to distributional shifts. 
Adversarial examples \citep{szegedy2013adv} are one such failure case: deep neural networks with nearly perfect performance provide incorrect predictions with very high confidence when evaluated on perturbations imperceptible to the human eye.
Adversarial examples are a serious hazard when deploying machine learning systems in security-sensitive applications.
More generally, deep learning systems quickly degrade in performance as the distributions of training and testing data differ slightly from each other \citep{ben2010theory}.

\begin{figure*}[t!]
        \centering
        \begin{subfigure}{0.30\textwidth}
            \centering
            \includegraphics[scale=0.45]{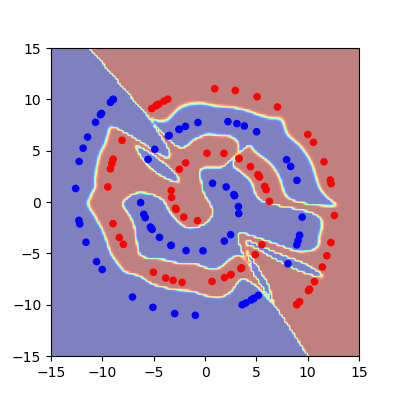}
            \caption[]%
            {}    
        \end{subfigure}
        \begin{subfigure}{0.30\textwidth}  
            \centering 
            \includegraphics[scale=0.45]{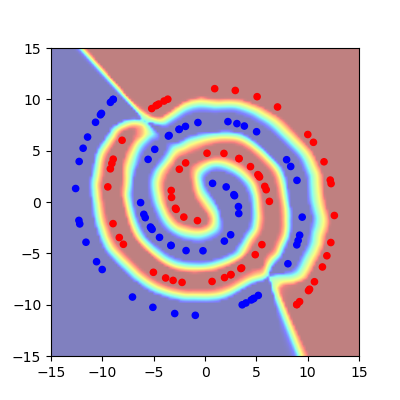}
            \caption[]%
            {}    
        \end{subfigure}
        \begin{subfigure}{0.30\textwidth}  
            \centering 
            \includegraphics[scale=0.45]{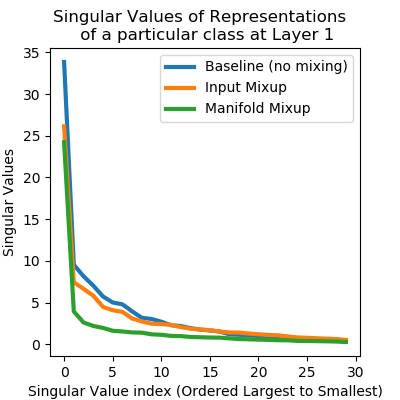}
            \caption[]%
            {}    
        \end{subfigure}
        \begin{subfigure}[b]{0.30\textwidth}   
            \centering 
            \includegraphics[scale=0.45]{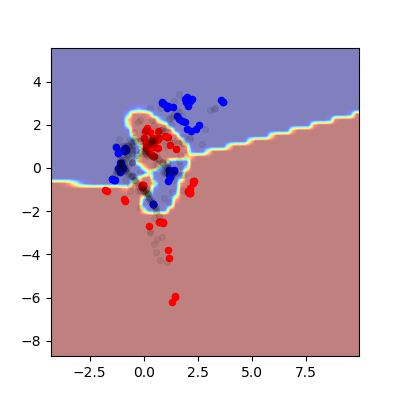}
            \caption[]%
            {}    
        \end{subfigure}
        \begin{subfigure}[b]{0.30\textwidth}   
            \centering 
            \includegraphics[scale=0.45]{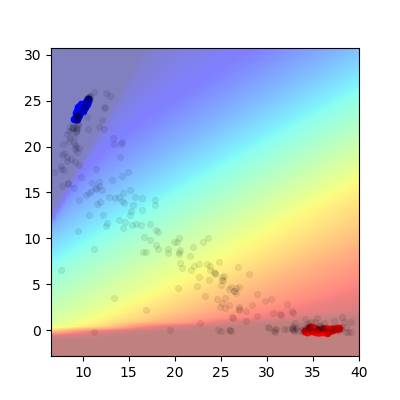}
            \caption[]%
            {}    
        \end{subfigure}
        \begin{subfigure}[b]{0.30\textwidth}  
            \centering 
            \includegraphics[scale=0.45]{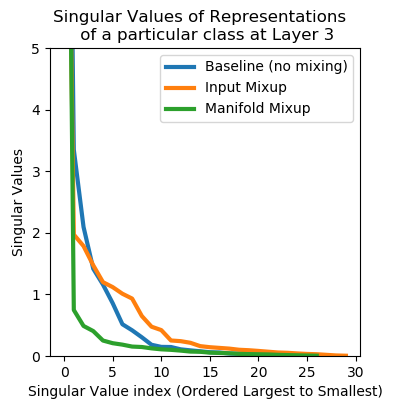}
            \caption[]%
            {}    
        \end{subfigure}
        
        \caption{An experiment on a network trained on the 2D spiral dataset with a 2D bottleneck hidden representation in the middle of the network. Manifold mixup has three effects on learning when compared to vanilla training. First, it smoothens decision boundaries (from a. to b.). Second, it improves the arrangement of hidden representations and encourages broader regions of low-confidence predictions (from d. to e.). Black dots are the hidden representation of the inputs sampled uniformly from the range of the input space.  Third, it flattens the representations (c. at layer 1, f. at layer 3). Figure \ref{appendix:figure:variousregs} shows that these effects are not accomplished by other well-studied regularizers (input mixup, weight decay, dropout, batch normalization, and adding noise to the hidden representations).}
        \label{fig:2dbottleneck}
\end{figure*}

In this paper, we realize several troubling properties concerning the hidden representations and decision boundaries of state-of-the-art neural networks.
First, we observe that the decision boundary is often sharp and close to the data.
Second, we observe that the vast majority of the hidden representation space corresponds to high confidence predictions, both on and off of the data manifold.

Motivated by these intuitions we propose \manifoldmixup{} (Section~\ref{sec:manifold_mixup}), a simple regularizer that addresses several of these flaws by training neural networks on linear combinations of hidden representations of training examples. 
Previous work, including the study of analogies through word embeddings (e.g. king $-$ man $+$ woman $\approx$ queen), has shown that interpolations are an effective way of combining factors \citep{mikolov2013wordembed}.  
Since high-level representations are often low-dimensional and useful to linear classifiers, linear interpolations of hidden representations should explore meaningful regions of the feature space effectively.
To use combinations of hidden representations of data as novel training signal, we also perform the same linear interpolation in the associated pair of one-hot labels, leading to mixed examples with soft targets.

To start off with the right intuitions, Figure~\ref{fig:2dbottleneck} illustrates the impact of \manifoldmixup{} on a simple two-dimensional classification task with small data.
In this example, vanilla training of a deep neural network leads to an irregular decision boundary (Figure~\ref{fig:2dbottleneck}a), and a complex arrangement of hidden representations (Figure~\ref{fig:2dbottleneck}d).
Moreover, every point in both the raw (Figure~\ref{fig:2dbottleneck}a) and hidden (Figure~\ref{fig:2dbottleneck}d) data representations is assigned a prediction with very high confidence.
This includes points (depicted in black) that correspond to inputs off the data manifold!  
In contrast, training the same deep neural network with \manifoldmixup{} leads to a smoother decision boundary (Figure~\ref{fig:2dbottleneck}b) and a simpler (linear) arrangement of hidden representations (Figure~\ref{fig:2dbottleneck}e).
In sum, the representations obtained by \manifoldmixup{} have two desirable properties: the class-representations are flattened into a minimal amount of directions of variation, and all points in-between these flat representations, most unobserved during training and off the data manifold, are assigned low-confidence predictions.
 
This example conveys the central message of this paper:
\begin{center}
    \emph{Manifold Mixup improves the hidden representations and decision boundaries of neural networks at multiple layers.}
\end{center}

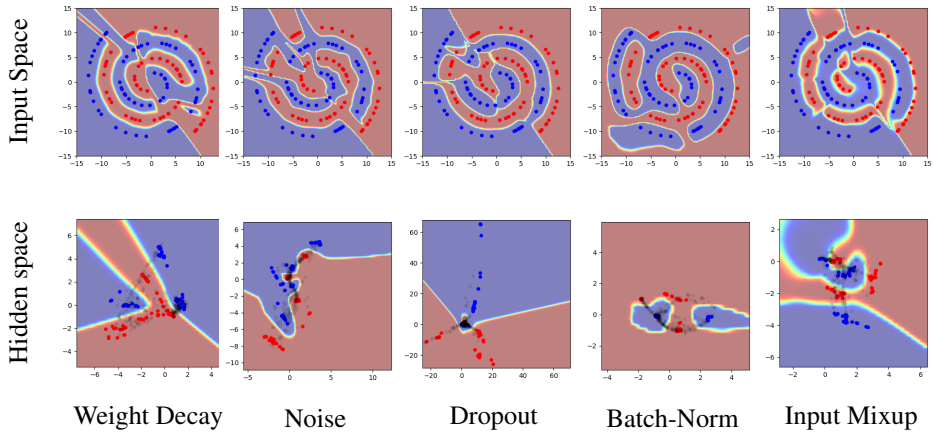
\begin{figure*}[t]
\centering
\begin{minipage}{0.17\textwidth}
\begin{tikzpicture}
  \node (img)  {\includegraphics[scale=0.25]{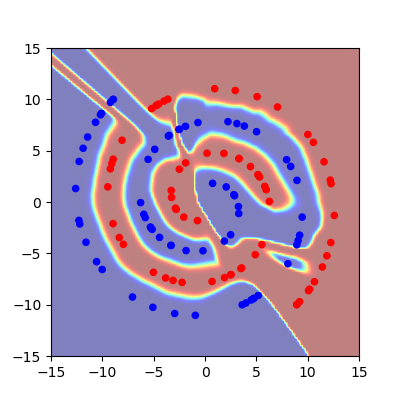}};
    \node[left=of img, node distance=0cm, rotate=90, anchor=center,yshift=-0.7cm] {Input Space};
    \node[left=of img, node distance=0cm, rotate=90, anchor=center,yshift=-0.7cm,font=\color{red}] {};
 \end{tikzpicture}
\end{minipage}%
\begin{minipage}{0.17\textwidth}
\begin{tikzpicture}
  \node (img)  {\includegraphics[scale=0.25]{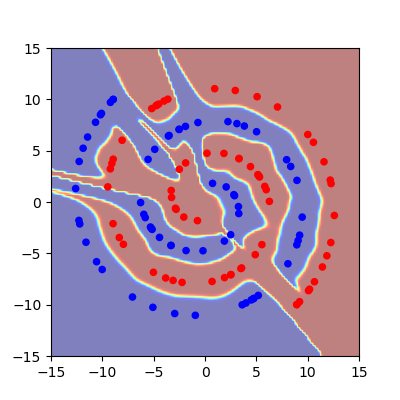}};
  \node[left=of img, node distance=0cm, rotate=90, anchor=center,yshift=-0.7cm,font=\color{red}] {};
\end{tikzpicture}
\end{minipage}%
\begin{minipage}{0.17\textwidth}
\begin{tikzpicture}
  \node (img)  {\includegraphics[scale=0.25]{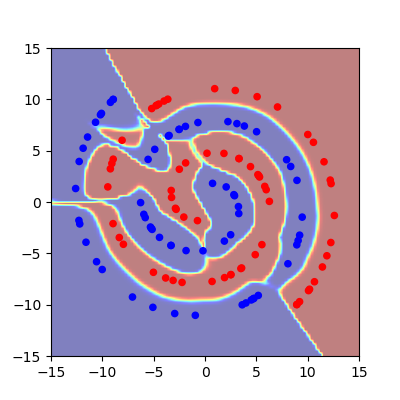}};
  \node[left=of img, node distance=0cm, rotate=90, anchor=center,yshift=-0.7cm,font=\color{red}] {};
\end{tikzpicture}
\end{minipage}%
\begin{minipage}{0.17\textwidth}
\begin{tikzpicture}
  \node (img)  {\includegraphics[scale=0.25]{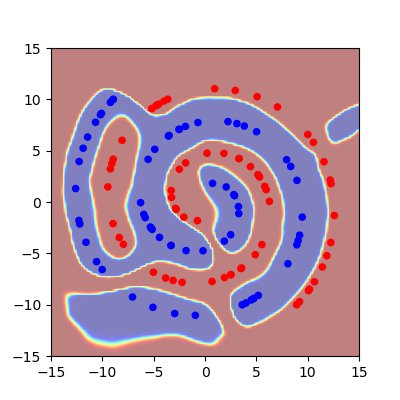}};
  \node[left=of img, node distance=0cm, rotate=90, anchor=center,yshift=-0.7cm,font=\color{red}] {};
\end{tikzpicture}
\end{minipage}%
\begin{minipage}{0.17\textwidth}
\begin{tikzpicture}
  \node (img)  {\includegraphics[scale=0.25]{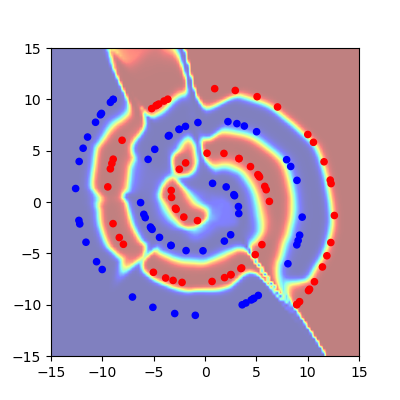}};
  \node[left=of img, node distance=0cm, rotate=90, anchor=center,yshift=-0.7cm,font=\color{red}] {};
\end{tikzpicture}
\end{minipage}%

\begin{minipage}{0.17\textwidth}
\begin{tikzpicture}
  \node (img)  {\includegraphics[scale=0.25]{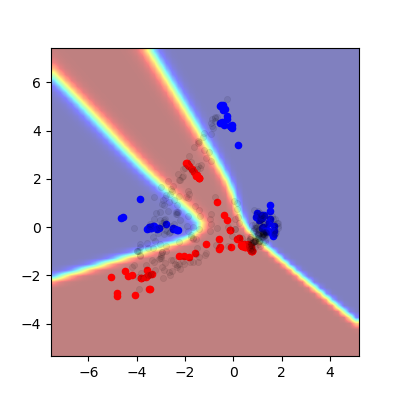}};
  \node[below=of img, node distance=0cm, yshift=1cm] {Weight Decay};
  \node[left=of img, node distance=0cm, rotate=90, anchor=center,yshift=-0.7cm] {Hidden space};
 \end{tikzpicture}
\end{minipage}%
\begin{minipage}{0.17\textwidth}
\begin{tikzpicture}
  \node (img)  {\includegraphics[scale=0.25]{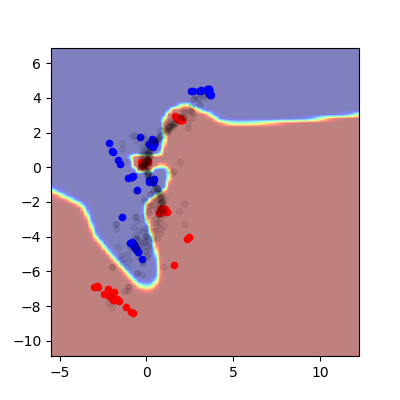}};
  \node[below=of img, node distance=0cm, yshift=1cm] {Noise};
  \node[left=of img, node distance=0cm, rotate=90, anchor=center,yshift=-0.7cm] {};
\end{tikzpicture}
\end{minipage}%
\begin{minipage}{0.17\textwidth}
\begin{tikzpicture}
  \node (img)  {\includegraphics[scale=0.25]{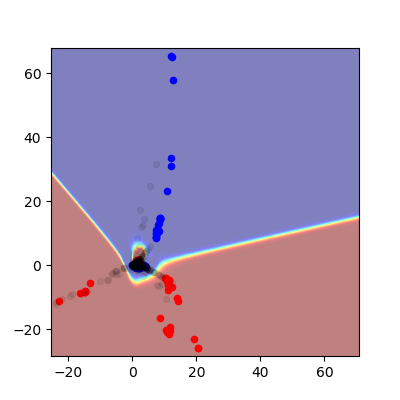}};
  \node[below=of img, node distance=0cm, yshift=1cm] {Dropout};
  \node[left=of img, node distance=0cm, rotate=90, anchor=center,yshift=-0.7cm] {};
\end{tikzpicture}
\end{minipage}%
\begin{minipage}{0.17\textwidth}
\begin{tikzpicture}
  \node (img)  {\includegraphics[scale=0.25]{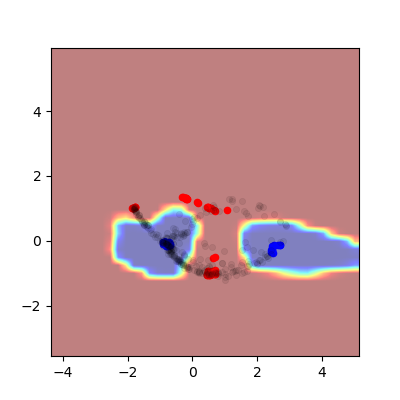}};
  \node[below=of img, node distance=0cm, yshift=1cm] {Batch-Norm};
  \node[left=of img, node distance=0cm, rotate=90, anchor=center,yshift=-0.7cm] {};
\end{tikzpicture}
\end{minipage}%
\begin{minipage}{0.17\textwidth}
\begin{tikzpicture}
  \node (img)  {\includegraphics[scale=0.25]{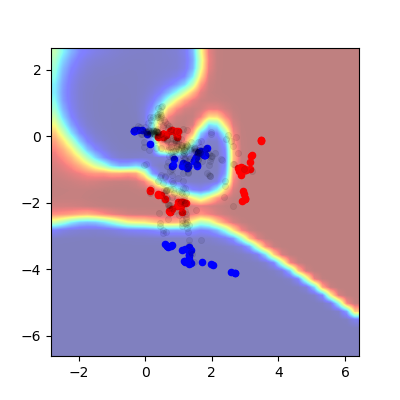}};
  \node[below=of img, node distance=0cm, yshift=1cm] {Input Mixup};
  \node[left=of img, node distance=0cm, rotate=90, anchor=center,yshift=-0.7cm] {};
\end{tikzpicture}
\end{minipage}%
\caption{The same experimental setup as Figure~\ref{fig:2dbottleneck}, but using a variety of competitive regularizers.  This shows that the effect of concentrating the hidden representation for each class and providing a broad region of low confidence between the regions is not accomplished by the other regularizers (although input space mixup does produce regions of low confidence, it does not flatten the class-specific state distribution).  Noise refers to gaussian noise in the input layer, dropout refers to dropout of 50\% in all layers except the bottleneck itself (due to its low dimensionality), and batch normalization refers to batch normalization in all layers.  }
\label{appendix:figure:variousregs}
\end{figure*}

More specifically, \manifoldmixup{} improves generalization in deep neural networks because it:
\begin{itemize}
    \item Leads to smoother decision boundaries that are further away from the training data, at multiple levels of representation.  Smoothness and margin are well-established factors of generalization \citep{bartlett1998generalization,lee1995vcsmooth}.  

    \item Leverages interpolations in deeper hidden layers, which capture higher level information \citep{zeiler2013visual} to provide additional training signal.

    \item Flattens the class-representations, reducing their number of directions with significant variance (Section~\ref{sec:flatten}). This can be seen as a form of compression, which is linked to generalization by a well-established theory \citep{tishby2015info,shwartz2017info} and extensive experimentation \citep{alemi2016bottleneck,belghazi2018mine,goyal2018transfer,achille2018information}.
\end{itemize}

Throughout a variety of experiments, we demonstrate four benefits of \manifoldmixup{}:

\begin{itemize}
    \item Better generalization than other competitive regularizers (such as Cutout, Mixup, AdaMix, and Dropout) (Section~\ref{sec:supervised}).

    \item Improved log-likelihood on test samples (Section~\ref{sec:supervised}).
    
    
    \item Increased performance at predicting data subject to novel deformations (Section~\ref{sec:deformations}).

    \item Improved robustness to single-step adversarial attacks.  This is the evidence  that \manifoldmixup{} pushes the decision boundary away from the data in some directions (Section~\ref{sec:adversarial}). This is not to be confused with full adversarial robustness, which is defined in terms of moving the decision boundary away from the data in \emph{all} directions.
\end{itemize}

\section{Manifold Mixup}
\label{sec:manifold_mixup}

Consider training a deep neural network $f(x) = f_k(g_k(x))$, where $g_k$ denotes the part of the neural network mapping the input data to the hidden representation at layer $k$, and $f_k$ denotes the part mapping such hidden representation to the output $f(x)$. 
Training $f$ using \manifoldmixup{} is performed in five steps.
First, we select a random layer $k$ from a set of eligible layers $\mathcal{S}$ in the neural network. This set may include the input layer $g_0(x)$.
Second, we process two random data minibatches $(x, y)$ and $(x', y')$ as usual, until reaching layer $k$.
This provides us with two intermediate minibatches $(g_k(x), y)$ and $(g_k(x'), y')$.
Third, we perform \inputmixup{} \citep{zhang2017mixup} on these intermediate minibatches.
This produces the mixed minibatch:
\begin{equation*}
    (\tilde{g}_k, \tilde{y}) := (\text{Mix}_\lambda(g_k(x), g_k(x')), \text{Mix}_\lambda(y, y')),
\end{equation*}
where $\text{Mix}_\lambda(a, b) = \lambda \cdot a + (1 - \lambda) \cdot b$.
Here, $(y, y')$ are one-hot labels, and the mixing coefficient $\lambda \sim \text{Beta}(\alpha, \alpha)$ as proposed in mixup \citep{zhang2017mixup}.
For instance, $\alpha = 1.0$ is equivalent to sampling $\lambda \sim U(0, 1)$. 
Fourth, we continue the forward pass in the network from layer $k$ until the output using the mixed minibatch $(\tilde{g}_k, \tilde{y})$.
Fifth, this output is used to compute the loss value and gradients that update all the parameters of the neural network.

Mathematically, \manifoldmixup{} minimizes:
\begin{align}
\label{eq:mixup_form}
    L(f) = &\expectation_{(x,y) \sim P}\,
    \expectation_{(x', y') \sim P}\,
    \expectation_{\lambda \sim \text{Beta}(\alpha, \alpha)}\,
    \expectation_{k \sim \mathcal{S}}
    \ell(f_k(\text{Mix}_\lambda(g_k(x), g_k(x'))), \text{Mix}_\lambda(y, y')).
\end{align}
Some implementation considerations.
We backpropagate gradients through the entire computational graph, including those layers before the mixup layer $k$ (Section \ref{sec:supervised} and appendix Section \ref{appendix:sec:learnedRepresentation} explore this issue in more detail).
In the case where $\mathcal{S} = \{0\}$, \manifoldmixup{} reduces to the original mixup algorithm of \citet{zhang2017mixup}.  

While one could try to reduce the variance of the gradient updates by sampling a random $(k, \lambda)$ per example, we opted for the simpler alternative of sampling a single $(k, \lambda)$ per minibatch, which in practice gives the same performance.
As in \inputmixup{}, we use a single minibatch to compute the mixed minibatch. We do so by mixing the minibatch with copy of itself with shuffled rows.



\begin{figure*}
    \centering
    \begin{subfigure}[b]{0.33\textwidth}
        \centering
        \begin{tikzpicture}[node distance=1cm, auto,]
            \node[very thick, fill=red!20, draw=red!50, circle] (B1) at (0,0) {B1};
            \node[very thick, fill=red!20, draw=red!50, circle] (B2) at (2,0) {B2};
            \node[very thick, fill=blue!20, draw=blue!50, circle] (A1) at (0,2) {A1};
            \node[very thick, fill=blue!20, draw=blue!50, circle] (A2) at (1.5,1.5) {A2};
            \draw[gray, dashed] (A1) -- (B1);
            \draw[gray, very thick] (A1) -- (B2);
            \draw[gray, dashed] (A2) -- (B1);
            \draw[gray, dashed] (A2) -- (B2);
            \draw[gray, dashed] (A1) -- (A2);
            \draw[gray, dashed] (B1) -- (B2);
            \node[very thick, fill=black!20, draw=black!50, circle] (dot) at (1, 1) {};
        \end{tikzpicture}
    \end{subfigure}%
    \begin{subfigure}[b]{0.33\textwidth}
        \centering
        \begin{tikzpicture}[node distance=1cm, auto,]
            \node[very thick, fill=red!20, draw=red!50, circle] (B1) at (0,0) {B1};
            \node[very thick, fill=red!20, draw=red!50, circle] (B2) at (2,0) {B2};
            \node[very thick, fill=blue!20, draw=blue!50, circle] (A1) at (0,2) {A1};
            \node[very thick, fill=blue!20, draw=blue!50, circle] (A2) at (1.5,1.5) {A2};
            \draw[gray, dashed] (A1) -- (B1);
            \draw[gray, dashed] (A1) -- (B2);
            \draw[gray, very thick] (A2) -- (B1);
            \draw[gray, dashed] (A2) -- (B2);
            \draw[gray, dashed] (A1) -- (A2);
            \draw[gray, dashed] (B1) -- (B2);
            \node[very thick, fill=black!20, draw=black!50, circle] (dot) at (0.9, 0.92) {};
        \end{tikzpicture}
    \end{subfigure}%
    \begin{subfigure}[b]{0.33\textwidth}
        \centering
        \begin{tikzpicture}[node distance=1cm, auto,]
            \node[very thick, fill=red!20, draw=red!50, circle] (B1) at (0,0) {B1};
            \node[very thick, fill=red!20, draw=red!50, circle] (B2) at (2,0) {B2};
            \node[very thick, fill=blue!20, draw=blue!50, circle] (A1) at (0,2) {A1};
            \node[very thick, fill=blue!20, draw=blue!50, circle] (A2) at (2,2) {A2};
            \draw[gray, dashed] (A1) -- (B1);
            \draw[gray, dashed] (A1) -- (B2);
            \draw[gray, dashed] (A2) -- (B1);
            \draw[gray, dashed] (A2) -- (B2);
            \draw[gray, dashed] (A1) -- (A2);
            \draw[gray, dashed] (B1) -- (B2);
            \node[very thick, fill=black!20, draw=black!50, circle] (dot) at (1, 1) {};
        \end{tikzpicture}
    \end{subfigure}
    \caption{Illustration on why Manifold Mixup learns flatter representations. The interpolation between A1 and B2 in the left panel soft-labels the black dot as 50\% red and 50\% blue, regardless of being very close to a blue point.  
    In the middle panel a different interpolation between A2 and B1 soft-labels the same point as ~95\% blue and ~5\% red.  
    However, since \emph{Manifold Mixup} \emph{learns} the hidden representations,
    the pressure to predict consistent soft-labels at interpolated points causes the states to become flattened (right panel).}
    \label{fig:intuitive}
\end{figure*}
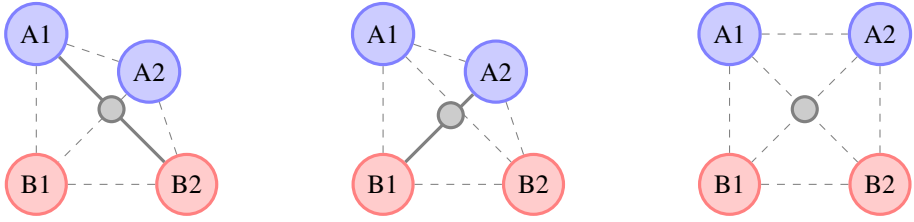


\section{Manifold Mixup Flattens Representations}
\label{sec:flatten}

We turn to the study of how \manifoldmixup{} impacts the hidden representations of a deep neural network.
At a high level, \manifoldmixup{} flattens the class-specific representations.
More specifically, this flattening reduces the number of directions with significant variance (akin to reducing their number of principal components).

In the sequel, we first prove a theory (Section~\ref{sec:theory}) that characterizes this behavior precisely under idealized conditions. 
Second, we show that this flattening also happens in practice, by performing the SVD of class-specific representations of neural networks trained on real datasets (Section~\ref{sec:empirical}).
%
Finally, we discuss why the flattening of class-specific representations is a desirable property (Section~\ref{sec:justification}).

\subsection{Theory}
\label{sec:theory}

We start by characterizing how the representations of a neural network are changed by \manifoldmixup{}, under a simplifying set of assumptions.
More concretely, we will show that if one performs mixup in a sufficiently deep hidden layer in a neural network, then the loss can be driven to zero if the dimensionality of that hidden layer $\mathrm{dim}\left(\mathcal{H}\right)$ is greater than the number of classes $d$.
As a consequence of this, the resulting representations for that class will have $\mathrm{dim}\left(\mathcal{H}\right)-d + 1$ dimensions.

A more intuitive and less formal version of this argument is given in Figure~\ref{fig:intuitive} and Appendix~\ref{appendix:sec:collision}.  

To this end, assume that $\mathcal{X}$ and $\mathcal{H}$ denote the input and representation spaces, respectively. We denote the label-set by $\mathcal{Y}$ and let $\mathcal{Z}=\mathcal{X}\times\mathcal{Y}$. Let $\mathcal{G}\subseteq\mathcal{H}^{\mathcal{X}}$ denote the set of functions realizable by the neural network, from the input to the representation. Similarly, let $\mathcal{F}\subseteq\mathcal{Y}^{\mathcal{H}}$ be the set of all functions realizable by the neural network, from the representation to the output.

We are interested in the solution of the following problem in some asymptotic regimes:
\begin{align}
\label{eq:amir:main}
    J(P) = &\inf_{g\in\mathcal{G},f\in\mathcal{F}}
    \expectation_{(x, y), (x', y'), \lambda}
    \ell(f(\text{Mix}_\lambda(g(x), g(x'))), \text{Mix}_\lambda(y, y')).
\end{align}
More specifically, let $P_D$ be the empirical distribution defined by a dataset $D = \{(x_i, y_i)\}_{i=1}^n$.
Then, let $f^\star \in \mathcal{F}$ and $g^\star \in \mathcal{G}$ be the minimizers of \eqref{eq:amir:main} for $P=P_D$.
Also, let $\mathcal{G}=\mathcal{H}^{\mathcal{X}}$, $\mathcal{F}=\mathcal{Y}^{\mathcal{H}}$, and $\mathcal{H}$ be a vector space.
These conditions \citep{cybenko1989approximation} state that the mappings realizable by large neural networks are dense in the set of all continuous bounded functions.
In this case, we show that the minimizer $f^\star$ is a linear function from $\mathcal{H}$ to $\mathcal{Y}$.
In this case, the objective \eqref{eq:amir:main} can be rewritten as:
%
%
\begin{align}
    J(P_D) &=
    \inf_{h_1, \ldots, h_n \in \mathcal{H}}
    \frac{1}{n\left(n-1\right)}\sum_{i\neq j}^{n}\left\{
    \inf_{f\in\mathcal{F}}
    \right.\label{obj2}
    \left.
    \int_{0}^{1}
    \ell(f(\text{Mix}_{\lambda}(h_i, h_j)), \text{Mix}_{\lambda}(y_i, y_j))\,
    p(\lambda)
    \mathrm{d}\lambda
    \right\},
    \nonumber
\end{align}
where $h_i = g(x_i)$.
\begin{thm2}
Let $\mathcal{H}$ be a vector space of dimension $\mathrm{dim}\left(\mathcal{H}\right)$, and let $d\in\mathbb{N}$ to represent the number classes contained in some dataset $D$. If $\dim\left(\mathcal{H}\right)\ge d-1$, then $J(P_D)=0$ and the corresponding minimizer $f^\star$ is a linear function from $\mathcal{H}$ to $\mathbb{R}^d$.
\label{thm:mainThm}
\end{thm2}
\begin{proof}
First, we observe that the following statement is true if $\dim\left(\mathcal{H}\right)\ge d-1$:
\begin{equation*}
\exists~
{A},{H}\in\mathbb{R}^{\mathrm{dim}\left(\mathcal{H}\right)\times d},
{b}\in\mathbb{R}^{d} : {A}^\top{H}+{b}{1}^\top_{d}=I_{d\times d},
\end{equation*}
where $I_{d\times d}$ and ${1}_d$ denote the $d$-dimensional identity matrix and all-one vector, respectively. In fact, ${b}{1}^\top_{d}$ is a rank-one matrix, while the rank of identity matrix is $d$. So, ${A}^\top{H}$ only needs rank $d-1$.

Let $f^\star(h) = {A}^\top{h}+{b}$ for all $h \in \mathcal{H}$.
Let $g^\star({x}_i)= {H}_{\zeta_i, :}$ be the $\zeta_i$-th column of $H$, where $\zeta_i\in\left\{1,\ldots,d\right\}$ stands for the class-index of the example $x_i$.
These choices minimize \eqref{eq:amir:main}, since:
\begin{align*}
    \ell(f^\star(\text{Mix}_{\lambda}(g^\star(x_i), g^\star(x_j))), \text{Mix}_{\lambda}(y_i, y_j)) &=\\
    \ell(A^\top\text{Mix}_{\lambda}(H_{\zeta_i,:}, H_{\zeta_j,:}) + b, \text{Mix}_{\lambda}(y_{i,\zeta_i}, y_{j,\zeta_j})) &=
    \ell(u, u) = 0.
\end{align*}
The result follows from ${A}^\top{H}_{\zeta_i,:}+{b}={y}_{i,\zeta_i}$ for all $i$.
\end{proof}

Furthermore, if $\dim\left(\mathcal{H}\right)>d-1$, then data points in the representation space $\mathcal{H}$ have some degrees of freedom to move independently.
\\[-2mm]
\begin{corl}
Consider the setting in Theorem \ref{thm:mainThm} with $\dim\left(\mathcal{H}\right)>d-1$. Let $g^\star\in\mathcal{G}$ minimize \eqref{eq:amir:main} under $P=P_D$. Then, the representations of the training points $g^\star({x}_i)$ fall on a $\left(\dim\left(\mathcal{H}\right)-d+1\right)$-dimensional subspace.
\end{corl}
\begin{proof}
From the proof of Theorem \ref{thm:mainThm}, ${A}^\top{H}=I_{d\times d}-{b}{1}^\top_{d}$.
The r.h.s. of this expression is a rank-$\left(d-1\right)$ matrix for a properly chosen ${b}$. Thus, ${A}$ can have a null-space of dimension $\dim\left(\mathcal{H}\right)-d+1$. This way, one can assign $g^\star({x}_i)={H}_{\zeta_i,:}+{e}_i$, where ${H}_{\zeta_i,:}$ is defined as in the proof of Theorem \ref{thm:mainThm}, and ${e}_i$ are arbitrary vectors in the null-space of ${A}$, for all $i=1,\ldots,n$.
\end{proof}

This result implies that if the \manifoldmixup{} loss is minimized, then the representation of each class lies on a subspace of dimension $\mathrm{dim}\left(\mathcal{H}\right)-d + 1$.  In the extreme case where $\mathrm{dim}\left(\mathcal{H}\right)= d - 1$, each class representation will collapse to a single point, meaning that hidden representations would not change in any direction, for each class-conditional manifold.  In the more general case  with larger $\mathrm{dim}\left(\mathcal{H}\right)$, the majority of directions in $\mathcal{H}$-space will be empty in the class-conditional manifold.

\subsection{Empirical Investigation of Flattening}
\label{sec:empirical}

We now show that the ``flattening'' theory that we have just developed also holds for real neural networks networks trained on real data.
To this end, we trained a collection of fully-connected neural networks on the MNIST dataset using multiple regularizers, including \manifoldmixup{}.
When using \manifoldmixup{}, we mixed representations at a single, fixed hidden layer per network.
After training, we performed the Singular Value Decomposition (SVD) of the hidden representations of each network, and analyzed their spectrum decay.

More specifically, we computed the largest singular value per class, as well as the sum of the all other singular values.
We computed these statistics at the first hidden layer for all networks and regularizers.
For the largest singular value, we obtained: 51.73 (baseline), 33.76 (weight decay), 28.83 (dropout), 33.46 (input mixup), and 31.65 (manifold mixup).
For the sum of all the other singular values, we obtained: 78.67 (baseline), 73.36 (weight decay), 77.47 (dropout), 66.89 (input mixup), and 40.98 (manifold mixup).
Therefore, weight decay, dropout, and input mixup all reduce the largest singular value, but only \manifoldmixup{} achieves a reduction of the sum of the all other singular values (e.g. flattening).
For more details regarding this experiment, consult Appendix~\ref{appendix:sec:flattening}.

%

\subsection{Why is Flattening Representations Desirable?}
\label{sec:justification}

We have presented evidence to conclude that \manifoldmixup{} leads to flatter class-specific representations, and that such flattening is not accomplished by other regularizers. 

But why is this flattening desirable?
First, it means that the hidden representations computed from our data occupy a much smaller volume.
Thus, a randomly sampled hidden representation within the convex hull spanned by the data in this space is more likely to have a classification score with lower confidence (higher entropy).
Second, compression has been linked to generalization in the information theory literature \citep{tishby2015info,shwartz2017info}.  Third compression has been been linked to generalization empirically in the past by work which minimizes mutual information between the features and the inputs as a regularizer \citep{belghazi2018mine,alemi2016bottleneck,achille2018information}.  

\section{Related Work}

Regularization is a major area of research in machine learning.
\manifoldmixup{} is a generalization of \inputmixup{}, the idea of building random interpolations between training examples and perform the same interpolation for their labels \citep{zhang2017mixup,tokozume2017betweenclass}.

Intriguingly, our experiments show that \manifoldmixup{} changes the representations associated to the layers before and after the mixing operation, and that this effect is crucial to achieve good results (Section \ref{sec:supervised}, Appendix \ref{appendix:sec:flattening}). This suggests that \manifoldmixup{} works differently than \inputmixup{}.  

Another line of research closely related to \manifoldmixup{} involves regularizing deep networks by perturbing their hidden representations.
These methods include dropout \citep{hinton2012dropout}, batch normalization \citep{ioffe2015bn}, and the information bottleneck \citep{alemi2016bottleneck}. Notably, \cite{hinton2012dropout} and \cite{ioffe2015bn} demonstrated that regularizers that work well in the input space can also be applied to the hidden layers of a deep network, often to further improve results.
We believe that \manifoldmixup{} is a complimentary form of regularization.  


\citet{zhao2018retrieve} explored improving adversarial robustness by classifying points using a function of the nearest neighbors in a fixed feature space.  This involves applying mixup between each set of nearest neighbor examples in that feature space.  The similarity between \citep{zhao2018retrieve} and \manifoldmixup{} is that both consider linear interpolations of hidden representations with the same interpolation applied to their labels.  However, an important difference is that \manifoldmixup{} backpropagates gradients through the earlier parts of the network (the layers before the point where mixup is applied), unlike \citep{zhao2018retrieve}.  In Section  \ref{sec:flatten} we explain how this discrepancy significantly affects the learning process.  

AdaMix \citep{guo2018adamix} is another related method which attempts to learn better mixing distributions to avoid overlap.  AdaMix performs interpolations only on the input space, reporting that their method degrades significantly when applied to hidden layers.  Thus, AdaMix may likely work for different reasons  than \manifoldmixup{}, and perhaps the two are complementary.
AgrLearn \citep{guo2018agrlearn} adds an information bottleneck layer to the output of deep neural networks.  AgrLearn leads to substantial improvements, achieving 2.45\% test error on CIFAR-10 when combined with Input Mixup \citep{zhang2017mixup}.  As AgrLearn is complimentary to Input Mixup, it may be also complimentary to \manifoldmixup{}.  \citet{wang2018interpolate} proposed an interpolation exclusively in the output space, does not backpropagate through the interpolation procedure, and has a very different framing in terms of the Euler-Lagrange equation (Equation 2) where the cost is based on unlabeled data (and the pseudolabels at those points) and the labeled data provide constraints.   


\section{Experiments}

We now turn to the empirical evaluation of \manifoldmixup{}.
We will study its regularization properties in supervised learning (Section~\ref{sec:supervised}),
as well as how it affects the robustness of neural networks to novel input deformations (Section~\ref{sec:deformations}), and adversarial examples (Section~\ref{sec:adversarial}). 

\subsection{Generalization on Supervised Learning}
\label{sec:supervised}

{\renewcommand{\arraystretch}{1.4}
\begin{table*}
           \centering
           \captionsetup[subtable]{position = top}
           \captionsetup[table]{position=top}
           \caption{Classification errors on (a) CIFAR-10 and (b) CIFAR-100.
                    We include results from \citep{zhang2017mixup}$\dagger$ and \citep{adamix}$\ddagger$.
                    Standard deviations over five repetitions.
                    }
           \label{tab:supervised}
           \begin{subtable}{0.5\linewidth}
               \resizebox{\linewidth}{!}{
               \begin{tabular}{lrr} 
               \toprule
               PreActResNet18 & \shortstack[l]{Test Error (\%)} & \shortstack[l]{Test NLL}  \\ 
               \midrule
               No Mixup & $4.83 \pm 0.066$ & $0.190 \pm 0.003$  \\ 
               \shortstack[l]{AdaMix$\ddagger$} & $3.52$ & NA \\ 
               \shortstack[l]{\inputmixup{}$\dagger$} & $4.20$ & NA \\ 
               \shortstack[l]{\inputmixup{} ($\alpha=1$)} & $3.82\pm 0.048$ & $0.186 \pm 0.004$\\
               \shortstack[l]{\manifoldmixup{} ($\alpha=2$)} & $\underline{2.95 \pm 0.046}$ & $\underline{0.137 \pm 0.003}$\\
               \midrule
               PreActResNet34 \\
               \midrule
               No Mixup & $4.64 \pm 0.072$ & $0.200 \pm 0.002$ \\ 
               \shortstack[l]{\inputmixup{} ($\alpha=1$)} & $2.88\pm 0.043$ & $0.176 \pm 0.002$  \\ 
               \shortstack[l]{\manifoldmixup{} ($\alpha=2$)} & $\underline{2.54 \pm 0.047}$ & $\underline{0.118 \pm 0.002}$\\  
               \midrule
               Wide-Resnet-28-10 \\
               \midrule
               No Mixup & $3.99\pm0.118$ & $0.162\pm 0.004$ \\ 
               \shortstack[l]{\inputmixup{} ($\alpha=1$)} & $2.92\pm 0.088$ & $0.173\pm 0.001$  \\ 
               \shortstack[l]{\manifoldmixup{} ($\alpha=2$)} & $\underline{2.55 \pm 0.024}$ & $\underline{0.111\pm 0.001}$\\  
               \bottomrule
               \end{tabular}
               }
               \caption{CIFAR-10}
               \label{tab:cifar10}
           \end{subtable}%
           \begin{subtable}{0.5\linewidth}
               \resizebox{\linewidth}{!}{
               \begin{tabular}{lrr} 
               \toprule
               PreActResNet18& \shortstack[l]{Test Error (\%)} & \shortstack[l]{Test NLL} \\ 
               \midrule
               No Mixup & $24.01\pm0.376$ & $1.189\pm0.002$ \\ 
               \shortstack[l]{AdaMix$\ddagger$} & 20.97 & n/a \\ 
               \shortstack[l]{\inputmixup{}$\dagger$} & 21.10 & n/a \\ 
               \shortstack[l]{\inputmixup{} ($\alpha=1$) } & $22.11\pm0.424$ & $1.055\pm0.006$ \\
               \manifoldmixup{} ($\alpha=2$) & {$\underline{20.34\pm0.525}$} & {$\underline{0.912\pm0.002}$} \\
               \midrule
               PreActResNet34 \\
               \midrule
               No Mixup & $23.55\pm0.399$ & $1.189\pm0.002 $ \\ 
               \inputmixup{} ($\alpha=1$) & $20.53\pm0.330$  & $1.039\pm0.045$ \\ 
               \manifoldmixup{} ($\alpha=2$) & $\underline{18.35\pm0.360}$ & $\underline{0.877\pm0.053}$ \\
               
               \midrule
               Wide-Resnet-28-10 \\
               \midrule
               No Mixup & $21.72\pm0.117$ & $1.023\pm0.004$ \\ 
               \shortstack[l]{\inputmixup{} ($\alpha=1$)} & $18.89\pm0.111$& $0.927\pm0.031 $ \\ 
               \shortstack[l]{\manifoldmixup{} ($\alpha=2$)} & $\underline{18.04\pm0.171}$ & $\underline{0.809\pm0.005}$\\  
               \bottomrule
               \end{tabular}
               }
               \caption{CIFAR-100}
               \label{tab:cifar100}
           \end{subtable}
           
\end{table*}
}

\begin{table}
\centering
\begin{minipage}{0.52\textwidth}
\caption{Classification errors and neg-log-likelihoods on SVHN.  We run each experiment five times.  }
\label{tab:svhn}
\resizebox{\linewidth}{!}{
\begin{tabular}{lrr} 
\toprule
PreActResNet18& \shortstack[l]{Test Error (\%)} & \shortstack[l]{Test NLL}\\ 
\midrule

No Mixup & $2.89\pm0.224$ & $0.136\pm0.001$ \\ 
Input Mixup ($\alpha=1$) & $2.76\pm0.014$ & $0.212\pm0.011$ \\
\manifoldmixup{} ($\alpha=2$) & \underline{$2.27\pm0.011$} & \underline{$0.122\pm0.006$} \\

\midrule
PreActResNet34 \\
\midrule
No Mixup & $2.97\pm0.004$ & $0.165\pm0.003$  \\ 
\inputmixup{} ($\alpha=1$) & $2.67\pm0.020$  & $0.199\pm0.009$ \\ 
\manifoldmixup{} ($\alpha=2$) & \underline{$2.18\pm0.004$} & \underline{$0.137\pm0.008$} \\
\midrule
Wide-Resnet-28-10 \\
\midrule
No Mixup & $2.80\pm0.044$ & $0.143\pm0.002$  \\ 
\inputmixup{} ($\alpha=1$) & $2.68\pm0.103$  & $0.184\pm0.022$ \\ 
\manifoldmixup{} ($\alpha=2$) & \underline{$2.06\pm0.068$} & \underline{$0.126\pm0.008$} \\
\bottomrule
\end{tabular}
}
\end{minipage}%
\hfill
\begin{minipage}{0.40\textwidth}
\centering
\caption{Accuracy on TinyImagenet.}
\label{tab:tinyimagenet}
\resizebox{\linewidth}{!}{
\begin{tabular}{lrr} 
\toprule
    PreActResNet18 & top-1 & top-5\\
\midrule
No Mixup & 55.52 & 71.04 \\ 
Input Mixup ($\alpha=0.2$) & $56.47$ & $71.74$ \\
Input Mixup ($\alpha=0.5$) & $55.49$ & $71.62$ \\
Input Mixup ($\alpha=1.0$) & $52.65$ & $70.70$ \\
Input Mixup ($\alpha=2.0$) & $44.18$ & $68.26$ \\
\manifoldmixup{} ($\alpha=0.2$) & $\underline{58.70}$ & $\underline{73.59}$ \\
\manifoldmixup{} ($\alpha=0.5$) & $57.24$ & $73.48$ \\
\manifoldmixup{} ($\alpha=1.0$) & $56.83$ & $73.75$ \\
\manifoldmixup{} ($\alpha=2.0$) & $48.14$ & $71.69$ \\
\bottomrule
\end{tabular}
}
\end{minipage}
\end{table}

We train a variety of residual networks \citep{he2016identity} using different regularizers: no regularization, AdaMix, \inputmixup{}, and \manifoldmixup{}.
We follow the training procedure of \citep{zhang2017mixup}, which is to use SGD with momentum, a weight decay of $10^{-4}$, and a step-wise learning rate decay.
Please refer to Appendix~\ref{appendix:sec:sl} for further details (including the values of the hyperparameter $\alpha$).
We show results for the CIFAR-10 (Table~\ref{tab:cifar10}), CIFAR-100 (Table~\ref{tab:cifar100}), SVHN (Table~\ref{tab:svhn}), and TinyImageNET (Table~\ref{tab:tinyimagenet}) datasets.
\manifoldmixup{} outperforms vanilla training, AdaMix, and \inputmixup{} across datasets and model architectures.
Furthermore, \manifoldmixup{} leads to models with significantly better Negative Log-Likelihood (NLL) on the test data.
In the case of CIFAR-10, \manifoldmixup{} models achieve as high as $50\%$ relative improvement of test NLL.



As a complimentary experiment to better understand why \manifoldmixup{} works, we zeroed gradient updates immediately after the layer where mixup is applied.
On the dataset CIFAR-10 and using a PreActResNet18, this led to a 4.33\% test error, which is worse than our results for \inputmixup{} and \manifoldmixup{}, yet better than the baseline.
Because \manifoldmixup{} selects the mixing layer at random, each layer is still being trained even when zeroing gradients, although it will receive less updates.
This demonstrates that \manifoldmixup{} improves performance by updating the layers both before and after the mixing operation.  


We also compared \manifoldmixup{} against other strong regularizers. For each regularizer, we selected the best hyper-parameters using a validation set. The training of PreActResNet50 on CIFAR-10 for 600 epochs led to the following test errors (\%): no regularization ($4.96 \pm 0.19$), Dropout ($5.09 \pm 0.09$), Cutout \citep{cutout} ($4.77 \pm 0.38$), Mixup ($4.25 \pm 0.11$), and Manifold Mixup ($3.77 \pm 0.18$). (Note that the results in Table~\ref{tab:supervised} for PreActResNet were run for 1200 epochs, and therefore are not directly comparable to the numbers in this paragraph.)  

To provide further evidence about the quality of representations learned with \manifoldmixup{}, we applied a $k$-nearest neighbour classifier on top of the features extracted from a PreActResNet18 trained on CIFAR-10.
We achieved test errors of 6.09\% (vanilla training), 5.54\% (\inputmixup{}), and 5.16\% (\manifoldmixup{}). 

Finally, we considered a synthetic dataset where the data generating process is a known function of disentangled factors of variation, and mixed in this space factors.  As shown in Appendix~\ref{appendix:sec:synthetic}, this led to significant improvements in performance. This suggests that mixing in the correct level of representation has a positive impact on the decision boundary. However, our purpose here is not to make any claim about when do deep networks learn representations corresponding to disentangled factors of variation.  

Finally, Table~\ref{tb:alpha} and Table~\ref{tb:layer_ablation} show the sensitivity of \manifoldmixup{} to the hyper-parameter $\alpha$ and the set of eligible layers $\mathcal{S}$.
(These results are based on training a PreActResNet18 for $2000$ epochs, so these numbers are not exactly comparable to the ones in Table \ref{tab:supervised}.) 
This shows that \manifoldmixup{} is robust with respect to choice of hyper-parameters, with improvements for many choices.


\begin{table*}[ht!]
\caption{Test accuracy on novel deformations.
         All models trained on normal CIFAR-100.}
\label{tab:deformation_tests}
\begin{tabular}{lrrrr} 
\toprule
Deformation & \shortstack{No Mixup} & \shortstack{Input Mixup\\($\alpha=1$)} & \shortstack{Input Mixup\\($\alpha=2$)} & \shortstack{\manifoldmixup{}\\ ($\alpha=2$)}  \\ 
\midrule
Rotation U(\ang{-20},\ang{20})      & $52.96$ & $55.55$ & 56.48 & $\underline{60.08}$  \\ 
Rotation U(\ang{-40},\ang{40})      & $33.82$ & $37.73$ & 36.78 & $\underline{42.13}$  \\  
Shearing U(\ang{-28.6}, \ang{28.6}) & $55.92$ & $58.16$ & 60.01 & $\underline{62.85}$ \\
Shearing U(\ang{-57.3}, \ang{57.3}) & $35.66$ & $39.34$ & 39.7 &  $\underline{44.27}$ \\
Zoom In (60\% rescale)              & $12.68$ & $\underline{13.75}$ & $13.12$ & $11.49$ \\
Zoom In (80\% rescale)              & $47.95$ & $52.18$ & 50.47 & $\underline{52.70}$ \\
Zoom Out (120\% rescale)            & $43.18$ & $60.02$ & 61.62 & $\underline{63.59}$ \\
Zoom Out (140\% rescale)            & $19.34$ & $41.81$ & 42.02 & $\underline{45.29}$\\

\bottomrule
\end{tabular}
\end{table*}

\begin{table}
\centering
\setlength\tabcolsep{4pt}
\begin{minipage}{0.48\textwidth}
\centering
\caption{Test accuracy \manifoldmixup{} for different sets of eligible layers $\mathcal{S}$ on CIFAR.}
\label{tb:layer_ablation}
\begin{tabular}{lrr} 
\toprule
$\mathcal{S}$ & CIFAR-10 & CIFAR-100\\ 
\midrule
$\{0,1,2\}$   & \underline{$97.23$} & $79.60$ \\
$\{0,1\}$     & $96.94$ & $78.93$ \\
$\{0,1,2,3\}$ & $96.92$ & \underline{$80.18$} \\
$\{1,2\}$     & $96.35$ & $78.69$ \\
$\{0\}$       & $96.73$ & $78.15$ \\
$\{1,2,3\}$   & $96.51$ & $79.31$ \\
$\{1\}$       & $96.10$ & $78.72$ \\
$\{2,3\}$     & $95.32$ & $76.46$ \\
$\{2\}$       & $95.19$ & $76.50$ \\
$\{ \}$       & $95.27$ & $76.40$ \\
\bottomrule
\end{tabular}

\end{minipage}%
\hfill
\begin{minipage}{0.48\textwidth}
\centering
\caption{Test accuracy (\%) of \inputmixup{} and \manifoldmixup{} for different $\alpha$ on CIFAR-10.}
\label{tb:alpha}
\begin{tabular}{rrr} 
\toprule
$\alpha$ & \inputmixup{} & \manifoldmixup{} \\ 
\midrule
$0.5$ & $96.68$ & \underline{$96.76$} \\
$1.0$ & $96.75$ & \underline{$97.00$}  \\
$1.2$ & $96.72$ & \underline{$97.03$}  \\
$1.5$ & $96.84$ & \underline{$97.10$} \\
$1.8$ & $96.80$ & \underline{$97.15$}  \\
$2.0$ & $96.73$ & \underline{$97.23$} \\
\bottomrule
\end{tabular}
 
\end{minipage}
\end{table}

\subsection{Generalization to Novel Deformations}
\label{sec:deformations}

To further evaluate the quality of representations learned with \manifoldmixup{}, we train PreActResNet34 models on the normal CIFAR-100 training split, but test them on novel (not seen during training) deformations of the test split.
These deformations include random rotations, random shearings, and different rescalings.
Better representations should generalize to a larger variety of deformations.
Table~\ref{tab:deformation_tests} shows that networks trained using \manifoldmixup{} are the most able to classify test instances subject to novel deformations, which suggests the learning of better representations.
For more results see Appendix \ref{appendix:sec:sl}, Table \ref{appendix:tb:deformation_tests}.

\subsection{Robustness to Adversarial Examples}
\label{sec:adversarial}

\begin{table}
\centering
\caption{
    Test accuracy on white-box FGSM adversarial examples on CIFAR-10 and CIFAR-100 (using a PreActResNet18 model) and SVHN (using a WideResNet20-10 model).
    We include the results of \citep{madry2017adv}$\dagger$.}
\label{tb:adversarial_examples}
\begin{tabular}{lrr} 
\toprule
CIFAR-10 & \shortstack{FGSM} \\ 
\midrule
No Mixup & 36.32   \\ 
\inputmixup{} ($\alpha=1$) & 71.51  \\ 
\shortstack[l]{\manifoldmixup{} ($\alpha=2$)} & \underline{77.50} \\ 
\shortstack[l]{PGD training (7-steps)$\dagger$} & 56.10 \\
\midrule
CIFAR-100 & \shortstack{FGSM} \\
\midrule
\inputmixup{} ($\alpha=1$) & 40.7 \\
\manifoldmixup{} ($\alpha=2$) & 44.96 \\
\bottomrule
SVHN & \shortstack{FGSM} \\
\midrule
\shortstack[l]{No Mixup} & 21.49  \\ 
\shortstack[l]{\inputmixup{} ($\alpha=1$)} & 56.98  \\ 
\shortstack[l]{\manifoldmixup{} ($\alpha=2$)} & 65.91 \\ 
\shortstack[l]{PGD training (7-steps)}$\dagger$ & \underline{72.80} \\ 
\bottomrule
\end{tabular}
\end{table}

Adversarial robustness is related to the position of the decision boundary relative to the data.
Because \manifoldmixup{} only considers some directions around data points (those corresponding to interpolations), we would not expect the model to be robust to adversarial attacks that consider any direction around each example.
However, since \manifoldmixup{} expands the set of examples seen during training, an intriguing hypothesis is that these expansions overlap with the set of possible adversarial examples, providing some degree of defense.
If this hypothesis is true, \manifoldmixup{} would force adversarial attacks to consider a wider set of directions, leading to a larger computational expense for the attacker.
To explore this, we consider the Fast Gradient Sign Method \citep[FGSM, ][]{goodfellow2014adv}, which constructs adversarial examples in one single step, 
thus considering a relatively small subset of directions around examples.
The performance of networks trained using \manifoldmixup{} against FGSM attacks is given in Table \ref{tb:adversarial_examples}.
One challenge in evaluating robustness against adversarial examples is the ``gradient masking problem'', in which a defense succeeds only by reducing the quality of the gradient signal. 
\cite{athalye2018obfuscate} explored this issue in depth, and proposed running an unbounded search for a large number of iterations to confirm the quality of the gradient signal.
\manifoldmixup{} passes this sanity check (consult Appendix \ref{appendix:sec:adversarial} for further details).
While we found that using \manifoldmixup{} improves the robustness to single-step FGSM attack (especially over \inputmixup{}), we found that \manifoldmixup{} did not significantly improve robustness against stronger, multi-step attacks such as PGD \citep{madry2017adv}.

\section{Connections to Neuroscience and Credit Assignment}

We present an intriguing connection between \manifoldmixup{} and a challenging problem in neuroscience.  At a high level, we can imagine systems in the brain which compute predictions from a stream of changing inputs, and pass these predictions onto other modules which return some kind of feedback signal~\citep{Lee-et-al-MLKDB2015-small,Scellier+Bengio-frontiers2017,Whittington+Bogacz-2018,Bartunov-et-al-2018}. For instance, these feedback signals can be gradients or targets for prediction.  There is a delay between the output of the prediction and the point in time in which the feedback can return to the system after having travelled across the brain.  Moreover, this delay could be noisy and could differ based on the type of the prediction or other conditions in the brain, as well as depending on which paths are considered (there are many skip connections between areas).  This means that it could be very difficult for a system in the brain to establish a clear correspondence between its outputs and the feedback signals that it receives over time.  

While it is preliminary, an intriguing hypothesis is that part of how systems in the brain could be working around this limitation is by averaging their states and feedback signals across multiple points in time.  The empirical results from mixup suggest that such a technique may not just allow successful computation, but also act as a potent regularizer.  \manifoldmixup{} strenghthens this result by showing that the same regularization effect can be achieved from mixing in higher level hidden representations.  

\section{Conclusion}

Deep neural networks often give incorrect, yet extremely confident predictions on examples that differ from those seen during training.
This problem is one of the most central challenges in deep learning.
We have investigated this issue from the perspective of the representations learned by deep neural networks.
We observed that vanilla neural networks spread the training data widely throughout the representation space, and assign high confidence predictions to almost the entire volume of representations.
This leads to major drawbacks since the network will provide high-confidence predictions to examples off the data manifold, thus lacking enough incentives to learn discriminative representations about the training data.
To address these issues, we introduced \manifoldmixup{}, a new algorithm to train neural networks on interpolations of hidden representations.
\manifoldmixup{} encourages the neural network to be uncertain across the volume of the representation space unseen during training.
This leads to concentrating the representations of the real training examples in a low dimensional subspace, resulting in more discriminative features.
Throughout a variety of experiments, we have shown that neural networks trained using \manifoldmixup{} have better generalization in terms of error and log-likelihood, as well as better robustness to novel deformations of the data and adversarial examples.
Being easy to implement and incurring little additional computational cost, we hope that \manifoldmixup{} will become a useful regularization tool for deep learning practitioners. 

\section*{Acknowledgements}
The authors thank Christopher Pal, Sherjil Ozair  and Dzmitry Bahdanau for useful discussions and feedback.
Vikas Verma was supported by Academy of Finland project 13312683 / Raiko Tapani AT kulut.
We would also like to acknowledge Compute Canada for providing computing resources used in this work.


\bibliography{icml2019}
\bibliographystyle{icml2019}

\clearpage
\newpage
\appendix

\section{Synthetic Experiments Analysis}
\label{appendix:sec:synthetic}
We conducted experiments using a generated synthetic dataset where each image is deterministically rendered from a set of independent factors.  The goal of this experiment is to study the impact of input mixup and an idealized version of \manifoldmixup{} where we know the true factors of variation in the data and we can do mixup in exactly the space of those factors.  This is not meant to be a fair evaluation or representation of how \manifoldmixup{} actually performs - rather it's meant to illustrate how generating relevant and semantically meaningful augmented data points can be much better than generating points by mixing in the input space.  

We considered three tasks.  In Task A, we train on images with angles uniformly sampled between (-\ang{70}, -\ang{50}) (label 0) with 50\% probability and uniformly between (\ang{50}, \ang{80}) (label 1) with 50\% probability.  At test time we sampled uniformly between (-\ang{30}, -\ang{10}) (label 0) with 50\% probability and uniformly between (\ang{10}, \ang{30}) (label 1) with 50\% probability.  Task B used the same setup as Task A for training, but the test instead used (-\ang{30}, -\ang{20}) as label 0 and (-\ang{10}, \ang{30}) as label 1.  In Task C we made the label whether the digit was a ``1'' or a ``7'', and our training images were uniformly sampled between (-\ang{70}, -\ang{50}) with 50\% probability and uniformly between (\ang{50}, \ang{80}) with 50\% probability.  The test data for Task C were uniformly sampled with angles from (-\ang{30}, \ang{30}).  

The examples of the data are in Figure \ref{fig:appendix:synth} and results are in Table \ref{appendix:tb:synthetic}.  In all cases we found that Input Mixup gave some improvements in likelihood but limited improvements in accuracy - suggesting that the even generating nonsensical points can help a classifier trained with Input Mixup to be better calibrated.  Nonetheless the improvements were much smaller than those achieved with mixing in the ground truth attribute space.  

\begin{figure*}[]

\centering
\includegraphics[width=0.3\textwidth,trim={0 14.0cm 0 0},clip]{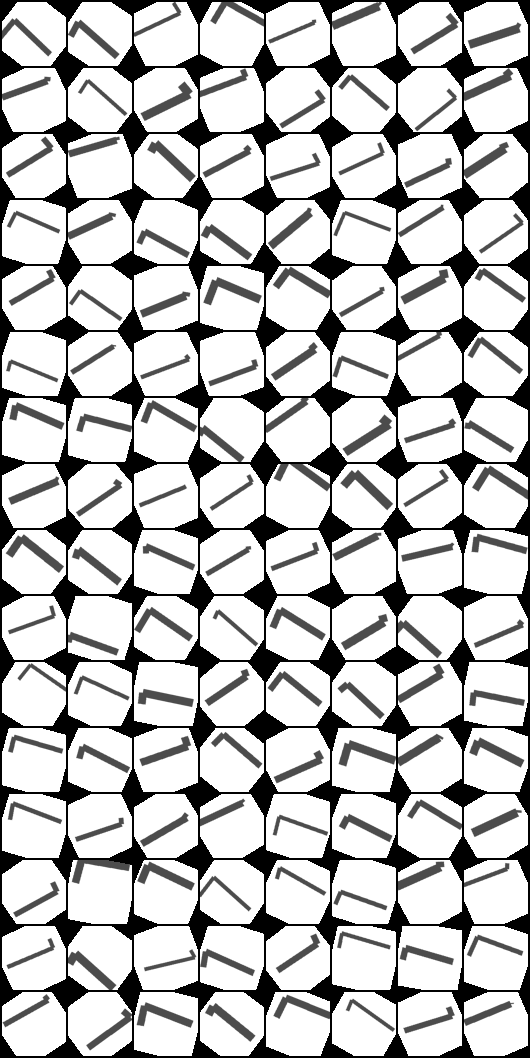}\hfill
\includegraphics[width=0.3\textwidth,trim={0 14.0cm 0 0},clip]{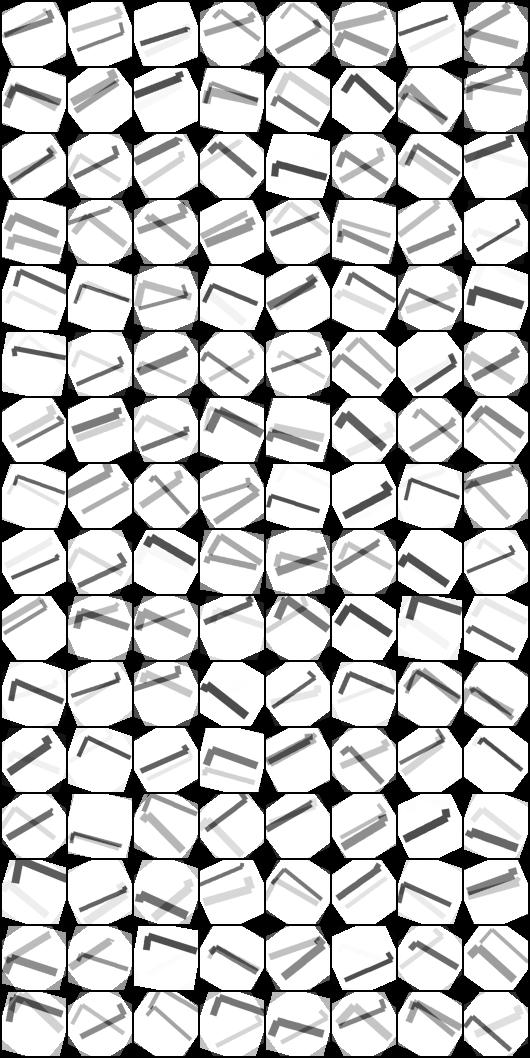}\hfill
\includegraphics[width=0.3\textwidth,trim={0 14.0cm 0 0},clip]{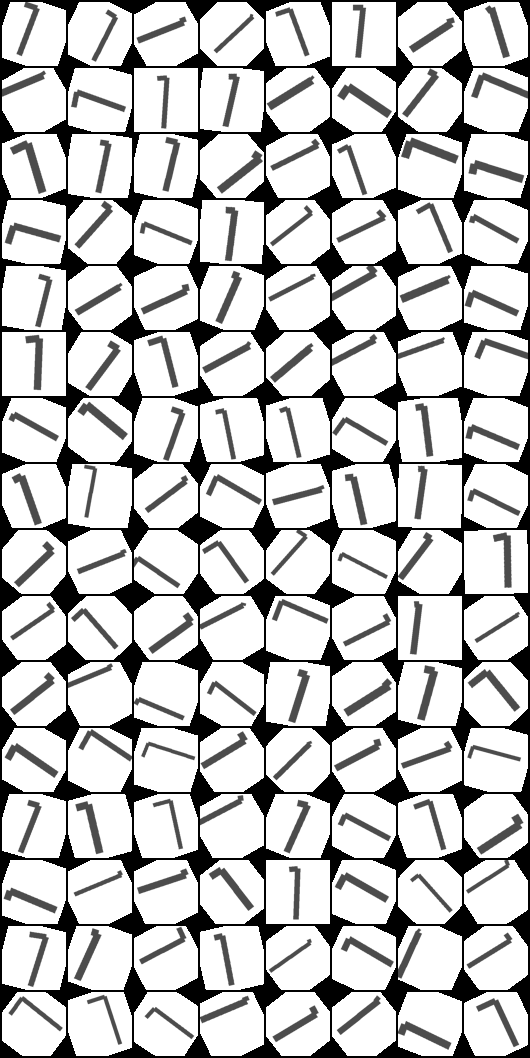}

\caption{Synthetic task where the underlying factors are known exactly.  Training images (left), images from input mixup (center), and images from mixing in the ground truth factor space (right).  }
\label{fig:appendix:synth}

\end{figure*}

{\renewcommand{\arraystretch}{1.2}
\begin{table*}[]
\centering
\caption{Results on synthetic data generalization task with an idealized Manifold Mixup (mixing in the true latent generative factors space).  Note that in all cases visible mixup significantly improved likelihood, but not to the same degree as factor mixup.  }
\label{appendix:tb:synthetic}
\begin{tabular}{llrr} 
\toprule
Task & Model & Test Accuracy & Test NLL  \\ 
\midrule
& No Mixup & 1.6 & 8.8310  \\ 
\textbf{Task A} & Input Mixup (1.0) & 0.0 & 6.0601 \\ 
& Ground Truth Factor Mixup (1.0) & 94.77 & 0.4940  \\ 
\midrule
& No Mixup & 21.25 & 7.0026 \\
\textbf{Task B} & Input Mixup (1.0) & 18.40 & 4.3149 \\
& Ground Truth Factor Mixup (1.0) & 84.02 & 0.4572 \\
\midrule
& No Mixup & 63.05 & 4.2871 \\
\textbf{Task C} & Input Mixup & 66.09 & 1.4181 \\
& Ground Truth Factor Mixup & 99.06 & 0.1279 \\
\bottomrule
\end{tabular}
\label{tab:synthetic}
\end{table*}
}


\section{Analysis of how \manifoldmixup{} changes learned representations}
\label{appendix:sec:learnedRepresentation}

We have found significant improvements from using \manifoldmixup{}, but a key question is whether the improvements come from changing the behavior of the layers before the mixup operation is applied or the layers after the mixup operation is applied.  This is a place where \manifoldmixup{} and Input Mixup are clearly differentiated, as Input Mixup has no ``layers before the mixup operation'' to change.  We conducted analytical experimented where the representations are low-dimensional enough to visualize.  More concretely, we trained a fully connected network on MNIST with two fully-connected leaky relu layers of 1024 units, followed by a 2-dimensional bottleneck layer, followed by two more fully-connected leaky-relu layers with 1024 units.  

We then considered training with no mixup, training with mixup in the input space, and training \textit{only} with mixup directly following the 2D bottleneck. We consistently found that \manifoldmixup{} has the effect of making the representations much tighter, with the real data occupying smaller region in the hidden space, and with a more well separated margin between the classes, as shown in Figure \ref{appendix:figure:mnist}

\begin{figure*}[htp]
\centering
\includegraphics[width=1.0\textwidth,trim={0 0cm 0 0},clip]{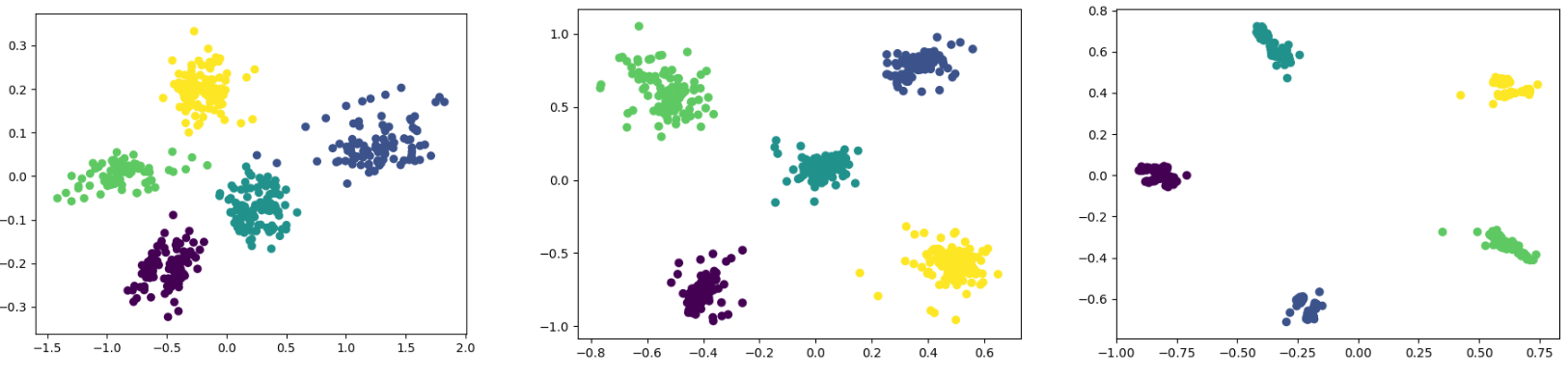}
\includegraphics[width=1.0\textwidth,trim={0 0cm 0 0},clip]{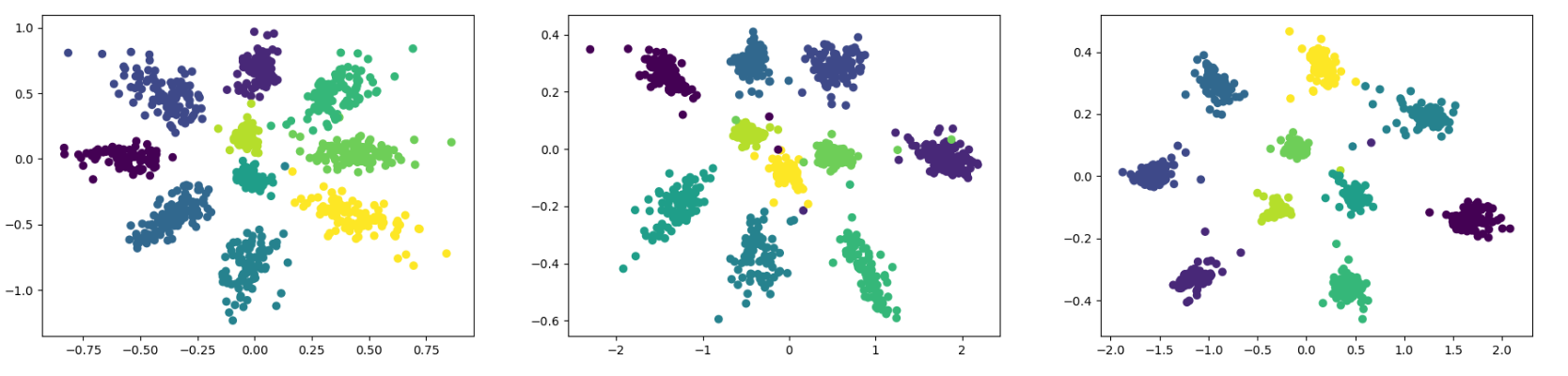}
\caption{Representations from a classifier on MNIST (top is trained on digits 0-4, bottom is trained on all digits) with a 2D bottleneck representation in the middle layer.  No Mixup Baseline (left), Input Mixup (center), \manifoldmixup{} (right).  }
\label{appendix:figure:mnist}
\end{figure*}

\section{Supervised Regularization Experimental Details}
\label{appendix:sec:sl}

{\renewcommand{\arraystretch}{1.2}
\begin{table*}[ht]
\centering
\caption{Models trained on the normal CIFAR-100 and evaluated on a test set with novel deformations.  \manifoldmixup{} (ours) consistently allows the model to be more robust to random shearing, rescaling, and rotation even though these deformations were not observed during training.  For the rotation experiment, each image is rotated with an angle uniformly sampled from the given range.  Likewise the shearing is performed with uniformly sampled angles.  Zooming-in refers to take a bounding box at the center of the image with k\% of the length and k\% of the width of the original image, and then expanding this image to fit the original size.  Likewise zooming-out refers to drawing a bounding box with k\% of the height and k\% of the width, and then taking this larger area and scaling it down to the original size of the image (the padding outside of the image is black).  }
\label{appendix:tb:deformation_tests}
\begin{tabular}{lrrrr} 
\toprule
Test Set Deformation & \shortstack{No Mixup\\Baseline} & \shortstack{Input Mixup\\$\alpha$=1.0} & \shortstack{Input Mixup\\$\alpha$=2.0} & \shortstack{\manifoldmixup{}\\$\alpha$=2.0}  \\ 
\midrule
Rotation U(\ang{-20},\ang{20}) & 52.96 & 55.55 & 56.48 & \textbf{60.08}  \\ 
Rotation U(\ang{-40},\ang{40}) & 33.82 & 37.73 & 36.78 & \textbf{42.13}  \\  
Rotation U(\ang{-60},\ang{60}) & 26.77 & 28.47 & 27.53 & \textbf{33.78} \\
Rotation U(\ang{-80},\ang{80}) & 24.19 & 26.72 & 25.34 & \textbf{29.95} \\
Shearing U(\ang{-28.6}, \ang{28.6}) & 55.92 & 58.16 & 60.01 & \textbf{62.85} \\
Shearing U(\ang{-57.3}, \ang{57.3}) & 35.66 & 39.34 & 39.7 & \textbf{44.27} \\
Shearing U(\ang{-114.6}, \ang{114.6}) & 19.57 & 22.94 & 22.8 & \textbf{24.69} \\
Shearing U(\ang{-143.2}, \ang{143.2}) & 17.55 & 21.66 & 21.22 & \textbf{23.56} \\
Shearing U(\ang{-171.9}, \ang{171.9}) & 22.38 & 25.53 & 25.27 & \textbf{28.02} \\
Zoom In (20\% rescale) & 2.43 & 1.9 & \textbf{2.45} & 2.03  \\
Zoom In (40\% rescale) & 4.97 & 4.47 & \textbf{5.23} & 4.17  \\
Zoom In (60\% rescale) & 12.68 & \textbf{13.75} & 13.12 & 11.49 \\
Zoom In (80\% rescale) & 47.95 & 52.18 & 50.47 & \textbf{52.7} \\
Zoom Out (120\% rescale) & 43.18 & 60.02 & 61.62 & \textbf{63.59} \\
Zoom Out (140\% rescale) & 19.34 & 41.81 & 42.02 & \textbf{45.29} \\
Zoom Out (160\% rescale) & 11.12 & 25.48 & 25.85 & \textbf{27.02} \\
Zoom Out (180\% rescale) & 7.98 & \textbf{18.11} & 18.02 & 15.68 \\
\bottomrule
\end{tabular}
\end{table*}
}


For supervised regularization we considered following architectures: PreActResNet18, PreActResNet34, and Wide-Resnet-28-10.  When using \manifoldmixup{}, we  selected the layer to perform mixing uniformly at random from a set of eligible layers.  In all our experiments, for the  PreActResNets architectures,
the eligible layers for mixing in  \manifoldmixup{} were : the input layer, the output from the first resblock, and the output from the second resblock.  For Wide-ResNet-20-10 architecture, the eligible layers for mixing in  \manifoldmixup{} were: the input layer and  the output from the first resblock.  For PreActResNet18, the first resblock has four layers and the second resblock has four layers.  For PreActResNet34, the first resblock has six layers and the second resblock has eight layers. For Wide-Resnet-28-10, the first resblock has four layers.  
Thus the mixing is often done fairly deep layers in the network. 

Throughout our experiments, we use SGD+Momentum optimizer with learning rate 0.1, momentum 0.9 and  weight-decay  $10^{-4}$, with step-wise learning rate decay.

For  Table \ref{tab:cifar10}, Table \ref{tab:cifar100} and Table \ref{tab:svhn}, we train the PreActResNet18, and PreActResNet34 for 1200 epochs with learning rate annealed by a factor of 10 at epoch 400 and 800. For above Tables, we train Wide-ResNet-28-10 for 400 epochs with learning rate annealed by a factor of 10 at epoch 200 and 300. In Table \ref{tab:tinyimagenet}, we train PreActResNet18 for 2000 epochs with learning rate annealed by a factor of 10 at epoch 1000 and 1500.

For Table \ref{tb:alpha} and Table \ref{tb:layer_ablation}, we train the PreActResNet18 network for 2000 epochs with learning rate annealed by a factor of 10 at epoch 1000 and 1500.

For Table \ref{tb:adversarial_examples},  Table \ref{tab:deformation_tests} and Table \ref{appendix:tb:deformation_tests}, we train the networks for 1200 epochs with  learning rate annealed by a factor of 10 at epoch 400 and 800.

In Figure \ref{appendix:fig:cifar10_train_loss} and Figure \ref{appendix:fig:cifar100_train_loss}, we present the training loss (Binary cross entropy) for CIFAR10 and CIFAR100 datasets respectively. We observe that performing \manifoldmixup{} in higher layers allows the train loss to go down faster as compared to the Input Mixup, which suggests that while Input Mixup may suffer from underfitting, \manifoldmixup{} alleviates this problem to some extend.

In Table \ref{appendix:tb:deformation_tests}, we present full set of experiments of Section \ref{sec:deformations}.
\subsection{Hyperparameter $\alpha$}
For Input Mixup on CIFAR10 and CIFAR100 datasets, we used the value $\alpha=1.0$ as recommended in \citep{zhang2017mixup}. For Input Mixup on SVHN and Tiny-imagenet datasets, we experimented with the $\alpha$ values in the set $\{0.1, 0.2, 0.4, 0.8. 1.0, 2.0, 4.0\}$. We obtained best results using $\alpha = 1.0$ and $\alpha = 0.2 $ for SVHN and Tiny-imagenet, respectively.

For \manifoldmixup{}, for all datasets, we experimented with the $\alpha$ values in the set $\{0.1, 0.2, 0.4, 0.8. 1.0, 2.0, 4.0\}$. We obtained best results with $\alpha=2.0$ for CIFAR10, CIFAR100 and SVHN and with $\alpha=0.2$ for Tiny-imagenet.



\begin{figure*}
  \centering
  \includegraphics[width=0.7\linewidth,trim={0 0cm 0 0},clip]{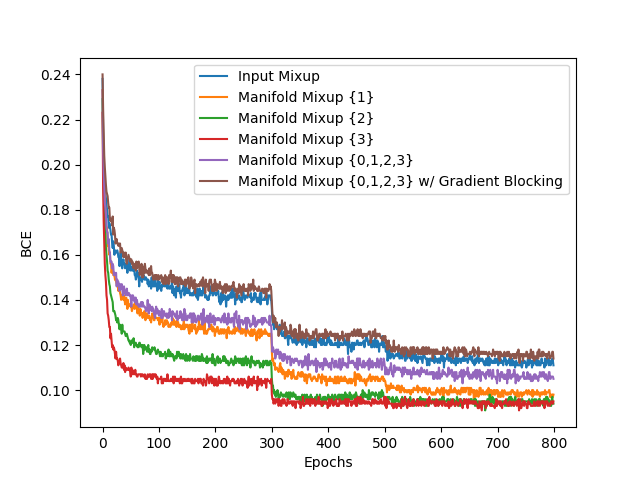}\hfill
  \caption{CIFAR-10 train set Binary Cross Entropy Loss (BCE) on Y-axis using PreActResNet18, with respect to training epochs (X-axis). The numbers in \{\} refer to the resblock  after which \manifoldmixup{} is performed.  The ordering of the losses is consistent over the course of training: Manifold Mixup with gradient blocked before the mixing layer has the highest training loss, followed by Input Mixup.  The lowest training loss is achieved by mixing in the deepest layer, which suggests that having more hidden units can help to prevent underfitting.  }
  \label{appendix:fig:cifar10_train_loss}
\end{figure*}

\begin{figure*}
  \centering
  \includegraphics[width=0.65\linewidth,trim={0 0cm 0 0},clip]{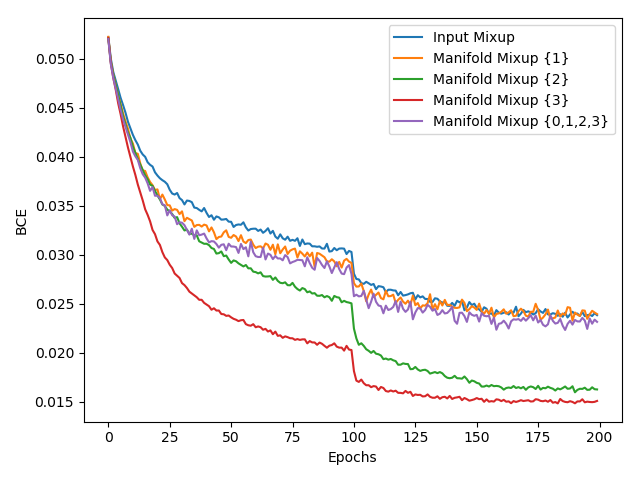}\hfill
  \caption{CIFAR-100 train set Binary Cross Entropy Loss (BCE) on Y-axis using PreActResNet50, with respect to training epochs (X-axis). The numbers in \{\} refer to the resblock  after which \manifoldmixup{} is performed. The lowest training loss is achieved by mixing in the deepest layer.  }
  \label{appendix:fig:cifar100_train_loss}
\end{figure*}

\section{Adversarial Examples}
\label{appendix:sec:adversarial}

We ran the unbounded projected gradient descent (PGD) \citep{madry2017adv} sanity check suggested in \citep{athalye2018obfuscate}.  We took our trained models for the input mixup baseline and manifold mixup and we ran PGD for 200 iterations with a step size of 0.01 which reduced the mixup model's accuracy to 1\% and reduced the \manifoldmixup{} model's accuracy to 0\%.  This is a evidence that our defense did not improve results primarily as a result of gradient masking.  


\section{Generative Adversarial Networks}
\label{appendix:sec:gan}
The recent literature has suggested that regularizing the discriminator is beneficial for training GANs \citep{improved_techniques, wgan, wgan_gp, spec_norm}. In a similar vein, one could add mixup to the original GAN training objective such that the extra data augmentation acts as a beneficial regularization to the discriminator, which is what was proposed in \cite{zhang2017mixup}. Mixup proposes the following objective\footnote{The formulation written is based on the official code provided with the paper, rather than the description in the paper. The discrepancy between the two is that the formulation in the paper only considers mixes between real and fake.}:
\begin{equation}
\max_{g} \min_{d} \mathbb{E}_{x_1,x_2,\lambda,z} \ \ell(d(\text{Mix}_\lambda(x_{1}, x_{2})), y(\lambda; x_{1}, x_{2})),
\label{eq:gan_mixup}
\end{equation}
where $x_{1}, x_{2}$ can be either real or fake samples, and $\lambda$ is sampled from a $Uniform(0,\alpha)$. Note that we have used a function $y(\lambda; x_{1}, x_{2})$ to denote the label since there are four possibilities depending on $x_{1}$ and $x_{2}$:
\begin{equation}
    y(\lambda; x_{1}, x_{2})= 
\begin{cases}
    \lambda,& \text{if } x_{1} \text{ is real and } x_{2} \text{ is fake} \\
    1-\lambda,      & \text{if } x_{1} \text{ is fake and } x_{2} \text{ is real} \\
    0, & \text{if both are fake} \\
    1, & \text{if both are real}
\end{cases}
\label{eq:lambda_cases}
\end{equation}

In practice however, we find that it did not make sense to create mixes between real and real where the label is set to 1, (as shown in equation \ref{eq:lambda_cases}), since the mixup of two real examples in input space is not a real example. So we only create mixes that are either real-fake, fake-real, or fake-fake. Secondly, instead of using just the equation in \ref{eq:gan_mixup}, we optimize it in addition to the regular minimax GAN equations:
\begin{equation}
\begin{split}
\max_{g} \min_{d} \mathbb{E}_{x} \ \ell(d(x), 1) + \mathbb{E}_{g(z)} \ \ell(d(g(z)), 0) +
\text{GAN mixup term (Equation \ref{eq:gan_mixup})} 
\end{split}
\end{equation}

Using similar notation to earlier in the paper, we present the manifold mixup version of our GAN objective in which we mix in the hidden space of the discriminator:
\begin{equation}\label{eq:gan}
\begin{split}
\min_{d} \mathbb{E}_{x_1,x_2,\lambda,z,k} \ \ell(d(x), 1) + \ell(d(g(z), 0) +
\ell(f_{k}(\text{Mix}_\lambda(g_{k}(x_{1}), g_{k}(x_{2}))), y(\lambda; x_{1}, x_{2})),
\end{split}
\end{equation}
where $g_k(\cdot)$ is a function denoting the intermediate output of the discriminator at layer $k$, and $f_{k}(\cdot)$ the output of the discriminator given input from layer $k$. 

The layer $k$ we choose the sample can be arbitrary combinations of the input layer (i.e., input mixup), or the first or second resblocks of the discriminator, all with equal probability of selection.

We run some experiments evaluating the quality of generated images on CIFAR10, using as a baseline JSGAN with spectral normalization \citep{spec_norm} (our configuration is almost identical to theirs). Results are averaged over at least three runs\footnote{Inception scores are typically reported with a mean and variance, though this is across multiple splits of samples across a single model. Since we run multiple experiments, we average their respective means and variances.}. From these results, the best-performing mixup experiments (both input and \manifoldmixup{}) is with $\alpha = 0.5$, with mixing in all layers (both resblocks and input) achieving an average Inception / FID of $8.04 \pm 0.08$ / $21.2 \pm 0.47$, input mixup achieving $8.03 \pm	0.08$	/ $21.4 \pm 0.56$, for the baseline experiment $7.97 \pm 0.07$ / $21.9 \pm 0.62$. This suggests that mixup acts as a useful regularization on the discriminator, which is even further improved by \manifoldmixup{}. (See Figure \ref{fig:appendix:gan_barplots} for the full set of experimental results.)


\begin{figure}[H]
  \centering
  \includegraphics[width=1.0\linewidth]{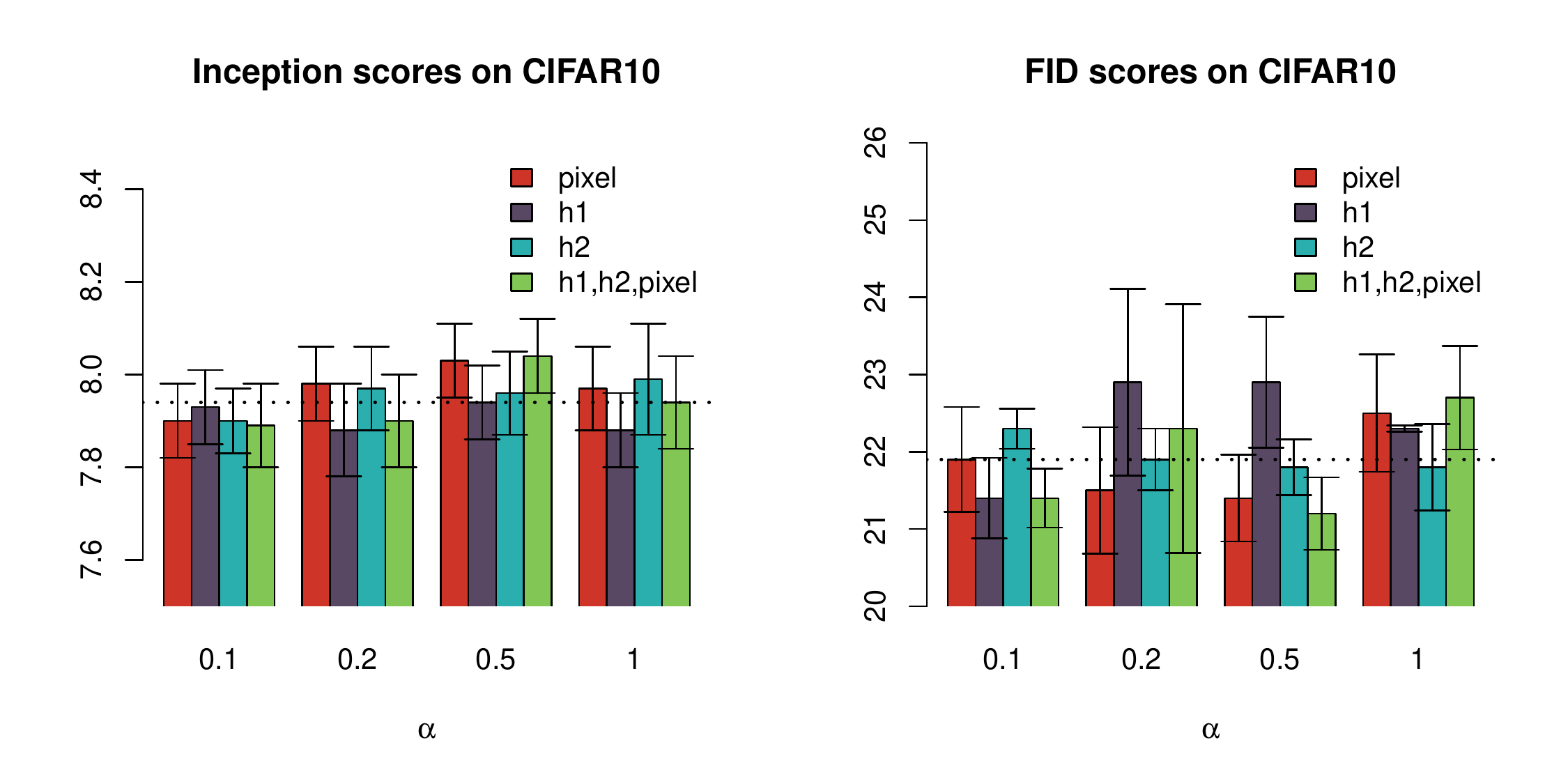}
  
  \caption{We test out various values of $\alpha$ in conjunction with either: input mixup ( \texttt{pixel}) \citep{zhang2017mixup}, mixing in the output of the first resblock (\texttt{h1}), mixing in either the output  of the first resblock or the output of the second resblock (\texttt{h1,2}), and mixing in the input or the output of the first resblock or the output of the second resblock (\texttt{1,2,pixel}). The dotted line indicates the baseline Inception / FID score. Higher scores are better for Inception, while lower is better for FID.}
  \label{fig:appendix:gan_barplots}
\end{figure}

\section{Intuitive Explanation of how Manifold Mixup avoids Inconsistent Interpolations}
\label{appendix:sec:collision}
An essential motivation behind manifold mixup is that as the network \textit{learns} the hidden states, it does so in a way that encourages them to be a flatter (per-class).  Section~\ref{sec:theory} characterized this for hidden states with any number of dimensions and Figure~\ref{fig:2dbottleneck} showed how this can occur on the 2D spiral dataset.  

Our goal here is to discuss concrete examples to illustrate why this flattening happens, as shown in Figure \ref{fig:intuitive}.  If we consider any two points, the interpolated point between them is based on a sampled $\lambda$ and the soft-target for that interpolated point is the targets interpolated with the same $\lambda$.  So if we consider two points A,B which have the same label, it is apparent that every point on the line between A and B should have that same label with 100\% confidence.  If we consider two points A,B with different labels, then the point which is halfway between them will be given the soft-label of 50\% the label of A and 50\% the label of B (and so on for other $\lambda$ values).  

It is clear that for many arrangements of data points, it is possible for a point in the space to be reached through distinct interpolations between different pairs of examples, and reached with different $\lambda$ values.  Because the learned model tries to capture the distribution $p(y \vert h)$, it can only assign a single distribution over the label values to a single particular point (for example it could say that a point is 100\% label A, or it could say that a point is 50\% label A and 50\% label B).  Intuitively, these inconsistent soft-labels at interpolated points can be avoided if the states for each class are more concentrated and the representations do not have variability in directions pointing towards other classes.  This leads to flattening: a reduction in the number of directions with variability.  The theory in Section~\ref{sec:theory} characterizes exactly what this concentration needs to be: that the representations for each class need to lie on a subspace of dimension equal to ``number of hidden dimensions'' - ``number of classes'' + 1.  


\section{Spectral Analysis of Learned Representations}
\label{appendix:sec:flattening}

When we refer to \textit{flattening}, we mean that the class-specific representations have reduced variability in some directions.  Our analysis in this section makes this more concrete.

We trained an MNIST classifier with a hidden state bottleneck in the middle with 12 units (intentionally selected to be just slightly greater than the number of classes).  We then took the representation for each class and computed a singular value decomposition (Figure~\ref{fig:appendix:svd_class_12} and Figure~\ref{fig:appendix:svd_class_30}) and we also computed an SVD over all of the representations together (Figure~\ref{fig:appendix:svd_all}).  Our architecture contained three hidden layers with 1024 units and LeakyReLU activation, followed by a bottleneck representation layer (with either 12 or 30 hidden units), followed by an additional four hidden layers each with 1024 units and LeakyReLU activation.  When we performed \manifoldmixup{} for our analysis, we only performed mixing in the bottleneck layer, and used a beta distribution with an alpha of 2.0.  Additionally we performed another experiment (Figure~\ref{fig:appendix:svd_early_layer} where we placed the bottleneck representation layer with 30 units immediately following the first hidden layer with 1024 units and LeakyReLU activation.  

We found that \manifoldmixup{} had a striking effect on the singular values, with most of the singular values becoming much smaller.  Effectively, this means that the representations for each class have variance in fewer directions.  While our theory in Section~\ref{sec:theory} showed that this flattening must force each classes representations onto a lower-dimensional subspace (and hence an upper bound on the number of singular values) but this explores how this occurs empirically and does not require the number of hidden dimensions to be so small that it can be manually visualized.  In our experiments we tried using 12 hidden units in the bottleneck Figure~\ref{fig:appendix:svd_class_12} as well as 30 hidden units Figure~\ref{fig:appendix:svd_class_30} in the bottleneck.  

Our results from this experiment are unequivocal: \manifoldmixup{} dramatically reduces the size of the smaller singular values for each classes representations.  This indicates a flattening of the class-specific representations.  At the same time, the singular values over all the representations are not changed in a clear way (Figure~\ref{fig:appendix:svd_all}), which suggests that this flattening occurs in directions which are distinct from the directions occupied by representations from other classes, which is the same intuition behind our theory.  Moreover, Figure~\ref{fig:appendix:svd_early_layer} shows that when the mixing is performed earlier in the network, there is still a flattening effect, though it is weaker than in the later layers, and again Input Mixup has an inconsistent effect.  

\begin{figure*}
  \centering
  \includegraphics[width=1.0\linewidth]{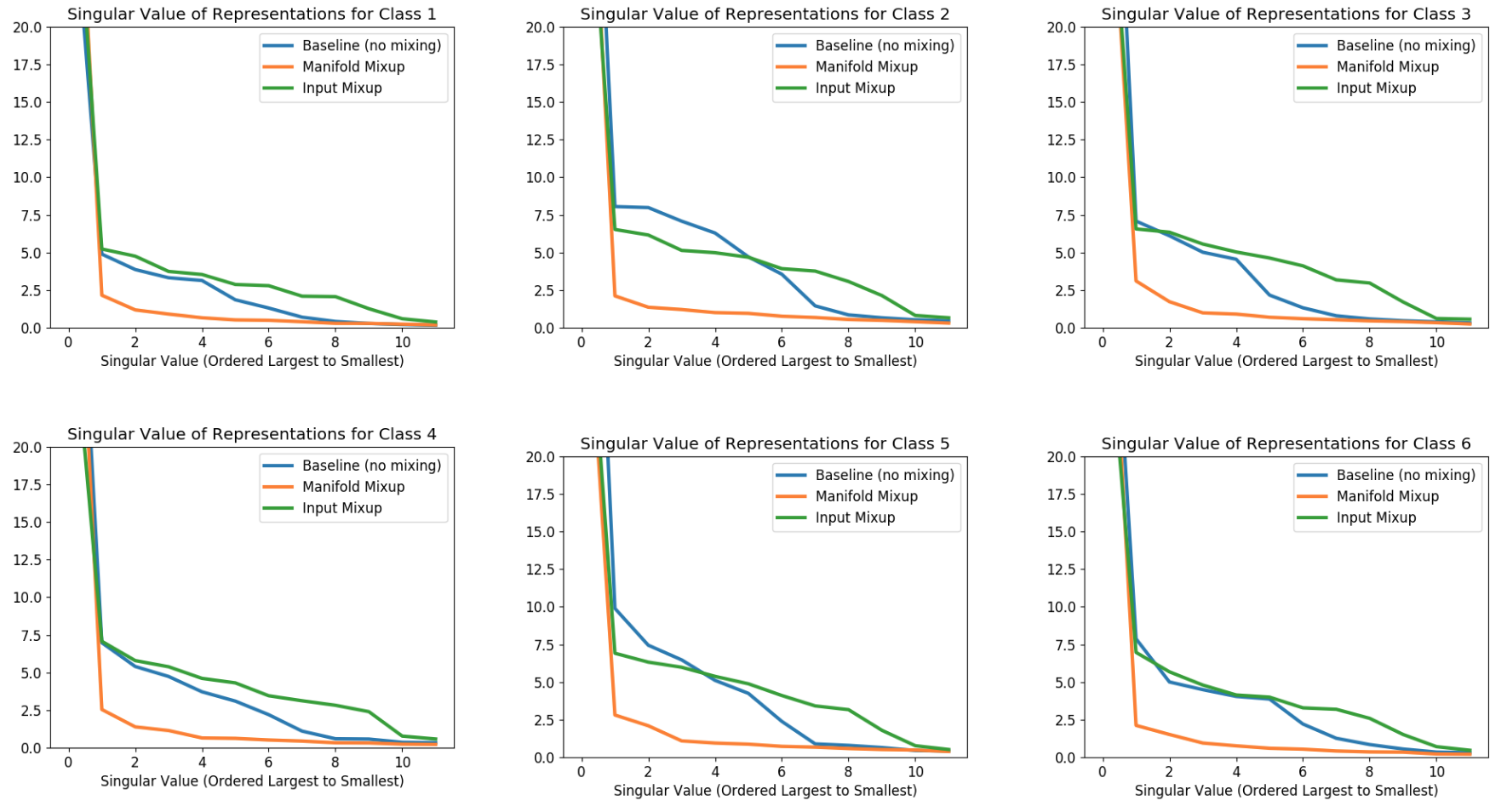}
  \caption{SVD on the class-specific representations in a bottleneck layer with 12 units following 3 hidden layers.  For the first singular value, the value (averaged across the plots) is 50.08 for the baseline, 37.17 for Input Mixup, and 43.44 for \manifoldmixup{} (these are the values at x=0 which are cutoff).  We can see that the class-specific SVD leads to singular values which are dramatically more concentrated when using \manifoldmixup{} with Input Mixup not having a consistent effect.  }
  \label{fig:appendix:svd_class_12}
\end{figure*}

\begin{figure*}
  \centering
  \includegraphics[width=1.0\linewidth]{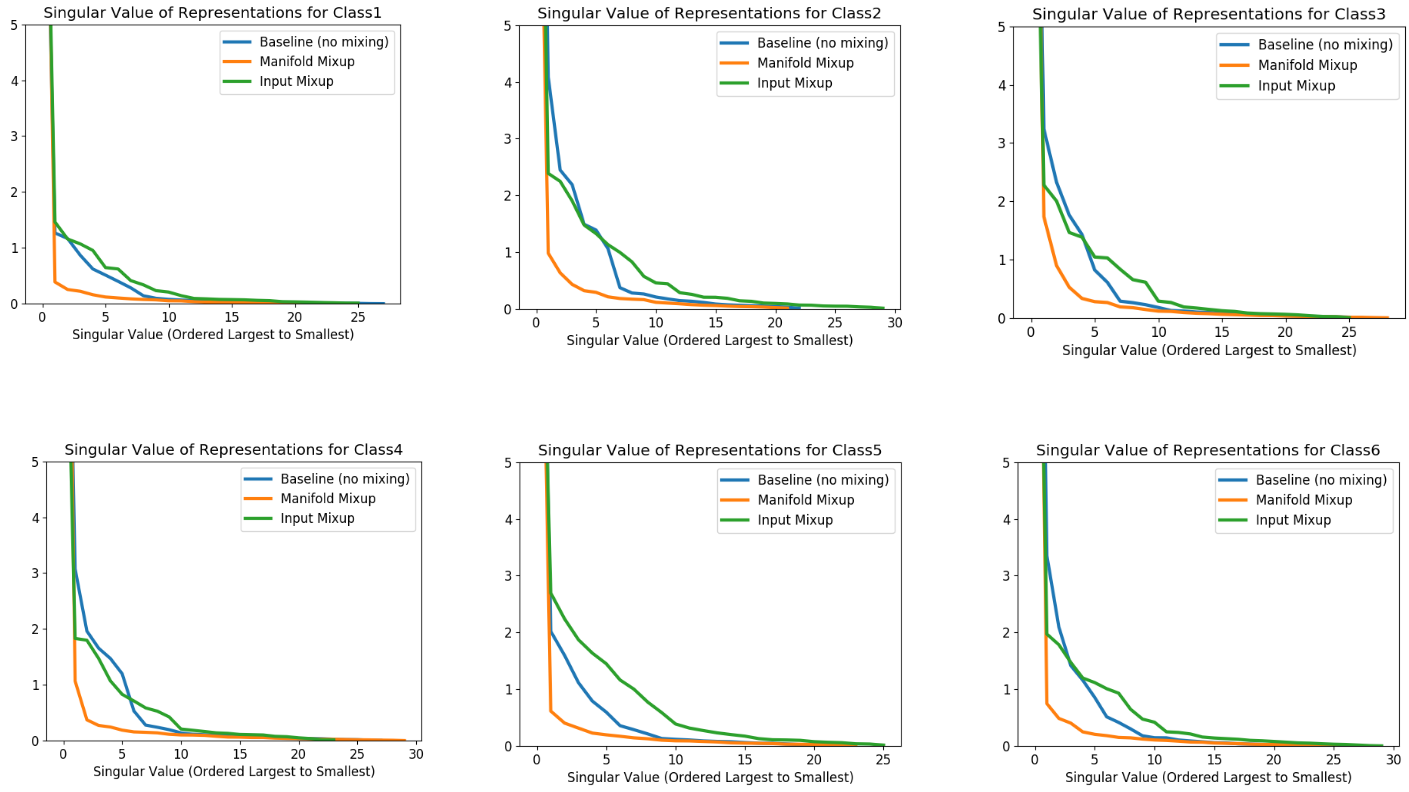}
  \caption{SVD on the class-specific representations in a bottleneck layer with 30 units following 3 hidden layers.  For the first singular value, the value (averaged across the plots) is 14.68 for the baseline, 12.49 for Input Mixup, and 14.43 for \manifoldmixup{} (these are the values at x=0 which are cutoff).  }
  \label{fig:appendix:svd_class_30}
\end{figure*}

\begin{figure*}
  \centering
  \includegraphics[width=1.0\linewidth]{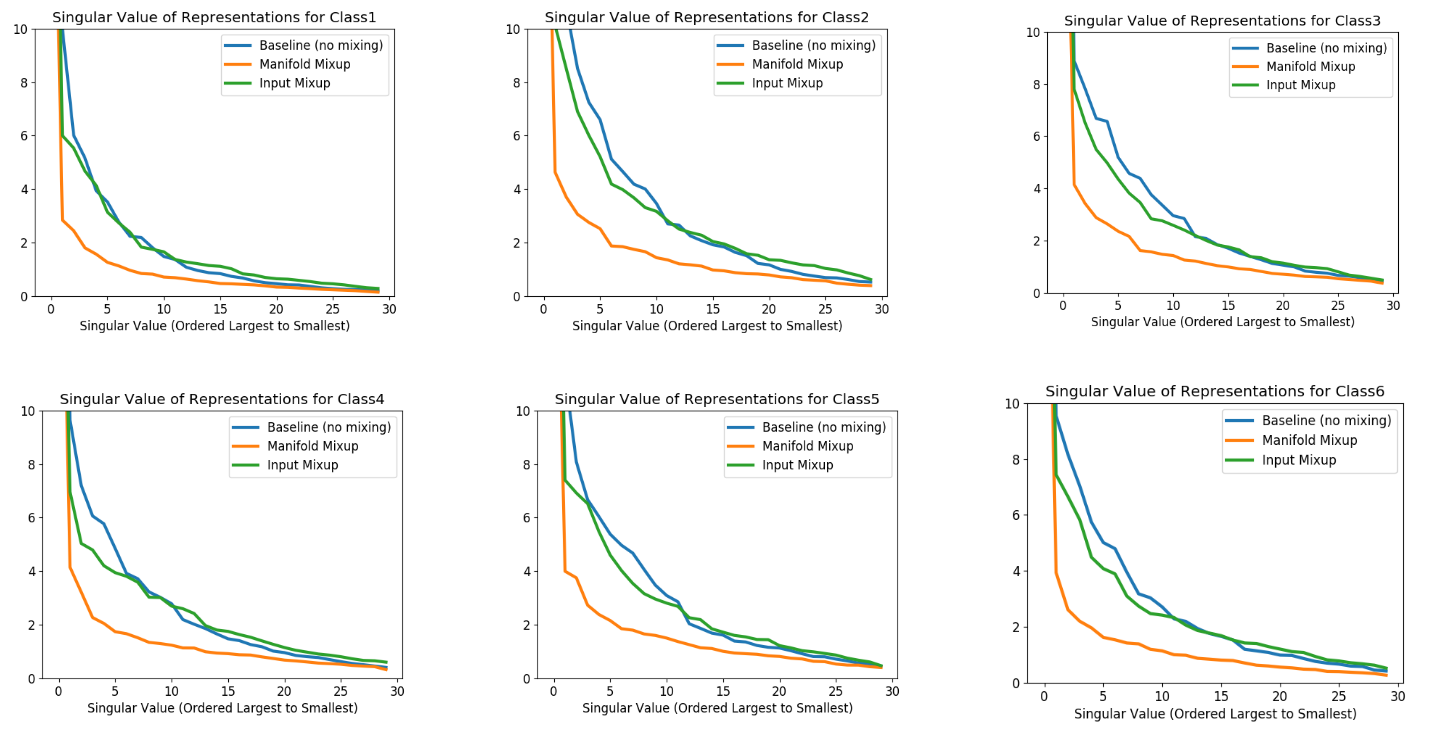}
  \caption{SVD on the class-specific representations in a bottleneck layer with 30 units following a single hidden layer.  For the first singular value, the value (averaged across the plots) is 33.64 for the baseline, 27.60 for Input Mixup, and 24.60 for \manifoldmixup{} (these are the values at x=0 which are cutoff).  We see that with the bottleneck layer placed earlier, the reduction in the singular values from \manifoldmixup{} is smaller but still clearly visible.  This makes sense, as it is not possible for this early layer to be perfectly discriminative.  }
  \label{fig:appendix:svd_early_layer}
\end{figure*}

\begin{figure*}[]
  \centering
  \includegraphics[width=0.7\linewidth,trim={0 0cm 0 0},clip]{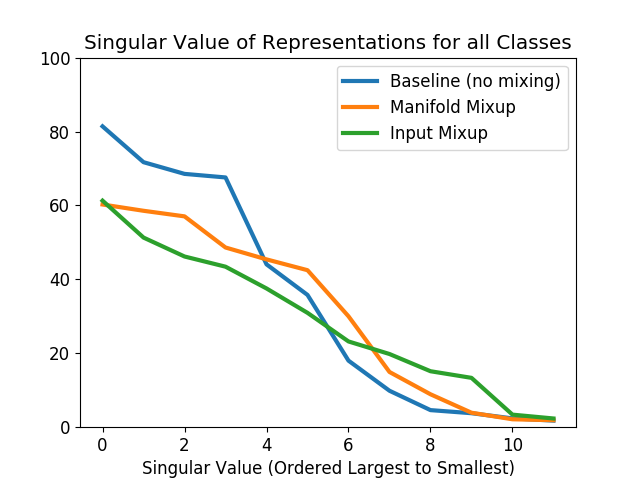}
  \caption{When we run SVD on all of the classes together (in the setup with 12 units in the bottleneck layer following 3 hidden layers), we see no clear difference in the singular values for the Baseline, Input Mixup, and \manifoldmixup{} models.  Thus we can see that the flattening effect of manifold mixup is entirely class-specific, and does not appear overall, which is consistent with what our theory has predicted.  More intuitively, this means that the directions which are being flattened are those directions which point towards the representations of different classes.  }
  \label{fig:appendix:svd_all}
\end{figure*}




\clearpage



\end{document}